\documentclass{article}

\usepackage{arxiv}
\usepackage[numbers,sort&compress]{natbib}
\usepackage[utf8]{inputenc}
\usepackage[T1]{fontenc}
\usepackage{hyperref}
\usepackage{url}
\usepackage{booktabs}
\usepackage{amsfonts}
\usepackage{nicefrac}
\usepackage{microtype}
\usepackage{graphicx}
\usepackage{subcaption}
\usepackage{xcolor}
\usepackage{amsmath}
\usepackage{amssymb}
\usepackage{mathtools}
\usepackage{amsthm}
\usepackage{makecell}
\usepackage{multirow}
\usepackage{tikz}
\usepackage{bm}
\usepackage{pgfplots}
\usepackage{enumitem}

\definecolor{latentred}{RGB}{210,45,45}
\definecolor{latentorange}{RGB}{230,135,35}
\definecolor{latentblue}{RGB}{55,95,180}
\definecolor{softgray}{RGB}{245,246,248}

\pgfplotsset{compat=1.18}
\usetikzlibrary{arrows.meta,positioning,calc,backgrounds,decorations.pathreplacing,fit}

\newtheorem{theorem}{Theorem}[section]
\newtheorem{proposition}[theorem]{Proposition}
\newtheorem{lemma}[theorem]{Lemma}
\newtheorem{corollary}[theorem]{Corollary}

\newcommand{\ve}[1]{\mathbf{#1}}

\newcommand{\vs}[1]{\boldsymbol{#1}}
\newcommand{\ad}[1]{#1^\dagger}

\newcommand{\rvector}[1]{\boldsymbol{\mathsf{#1}}}

\newcommand{\identity}{\mathbb{I}}

\newcommand{\filter}{\mathbf{w}_0}

\title{Blind Recovery of Latent Domains via Unsupervised Symmetry Discovery}

\author{
  Onur Efe \\
  Physics Department, Bogazici University \\
  \texttt{onur.efe@bogazici.edu.tr}
  \And
  Arkadas Ozakin \\
  Physics Department, Bogazici University \\
  \texttt{arkadas@bogazici.edu.tr}
}

\begin{document}
\maketitle

\begin{abstract}
Primary motivation in blind inverse problems is to recover signals of interest from corrupted observations without knowing the obfuscating mechanism. Blind deconvolution is a prominent approach when the corruption is convolutional, but it is not applicable when general linear transformations obfuscate the domain structure. In this work, we propose an unsupervised framework for recovering latent domains and signals by discovering symmetries of the data distribution. Our framework models observations as linear measurements of signals sampled from a latent random field, and optimizes a shallow group-convolutional network by imposing stationarity and locality regularization at the model output. The model learns a latent symmetry action and an appropriate filter, thereby mapping unstructured observations to a symmetry-based representation that reveals latent signals. Experiments on stochastic processes, Ising models, shuffled and bit-scrambled images, and neural recordings show that the method recovers latent domains and signals from unstructured observations, suggesting symmetry discovery as a new direction for unsupervised structure learning and blind inverse problems.
\end{abstract}

\section{Introduction}\label{section:intro}

Real-world observations often contain obfuscated signals due to imperfect sensor responses or physical corruptions. In many cases, the corruption mechanism is only partially known, making it difficult to recover the underlying signal even when the relevant information is present in its obfuscated form. When the corruption is linear and admits a convolutional form, the resulting inverse problem can be addressed using blind deconvolution algorithms~\cite{corbetta2022blind}. 

However, unknown transformations---such as permutations in ad-hoc sensor networks~\cite{balzano2007blind}, bit-scrambling in image encryption~\cite{ye2010image}, and scatterers in imaging~\cite{bertolotti2012non}---can destroy the apparent geometry of the observations preventing a convolutional formulation, thereby giving rise to challenging inverse problems. In such cases, recovering the latent signal also requires recovering the domain on which the signal is defined.

We propose that symmetry discovery provides a principled mechanism for recovering latent domains by exploiting the isomorphism between regular group actions and the domains on which they act. In particular, translation symmetries over spatiotemporal domains provide a natural instance of this principle. We explore this connection through an unsupervised approach that discovers generalized translation symmetry representations and uses them as a coordinate system for accessing the latent domain structure.

Our formulation relies on modeling the latent domain as a locally correlated stationary random field, and we assume that data samples arise by observing latent signals through an unknown linear measurement operator. Specifically, we optimize a group-convolutional layer with a learnable symmetry action and filter, and impose stationarity and locality regularizations on the model output. This formulation enables recovery of latent signals by discovering both the translation action and an appropriate filter, up to unavoidable ambiguities in the latent coordinate system such as rotations, reflections, and axis ordering.

Across all experiments, the model operates directly on unordered vector observations. First, we demonstrate domain recovery and symmetry discovery under dense linear transformations using stochastic processes and Ising models. We then revisit the permutation recovery problem~\citep{roux2007learning} by training the model on pixel-shuffled MNIST, and further evaluate robustness in a bit-scrambled MNIST setting inspired by image-encryption transformations~\cite{ye2010image}. Although bit-level scrambling is nonlinear at the pixel-intensity level, the task can be viewed as a linear inverse problem after representing each image as a high-dimensional binary vector with shuffled components.

Finally, we apply the model to real-world neural recordings, where it discovers a latent rotation operator on 110-dimensional spike-rate vectors and learns an equivariant representation aligned with visual stimuli. In this setting, the model acts as an unsupervised neural decoder, recovering stimulus orientation up to a constant phase. 

\begin{figure}[!t]
\centering
\includegraphics[trim={0cm 0cm 0cm 0cm}, clip, width=0.5\linewidth]{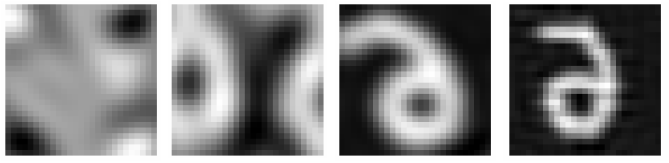}
\caption{\textbf{Bit-scrambled MNIST recovery.} The model is trained on randomly shuffled bits obtained from image pixels. The underlying digit, up to reflection, is recovered by discovering the translation symmetry action and an appropriate filter.}
\label{figure:2d-permutation-group-convolution-tensor-snapshots}
\end{figure}

\noindent\textbf{Our main contributions are:}

\begin{itemize}[leftmargin=*]

\item \textbf{Formulating blind domain recovery:}
We propose a generalized linear measurement model that interprets unstructured vector-valued observations as linear measurements of continuous latent signals. Under this model, we show that unstructured observations are linked to the latent signals through symmetries acting regularly on the latent domain (Section~\ref{sec:problem-setting}).

\item \textbf{Domain recovery through symmetry discovery:}
We propose a framework for recovering latent domains under dense linear transformations by discovering symmetries of the data distribution, and support the methodology through identifiability arguments (Section~\ref{sec:identifiability}).

\item \textbf{Flexible symmetry discovery in augmented spaces:}
We define symmetry actions in an embedding space, enabling recovery without assuming explicit sample coordinates or native symmetry representations in the original observation space~\cite{desai2022symmetry, yang2023latent}. The method applies to high-dimensional unstructured observations, including spike-rate vectors and shuffled image datasets (Section~\ref{sec:experiments}).

\end{itemize}

Our implementation is publicly available at \href{https://github.com/onurefe/blind-domain-recovery.git}{GitHub}.
\section{Related work}\label{section:related_work}

\textbf{Blind inverse problems.}
Blind inverse problems seek to invert unknown forward operators, typically requiring restrictive priors such as sparsity~\citep{ongie2020deeplearningtechniquesinverse}. Classical approaches like blind deconvolution~\citep{corbetta2022blind} successfully address these challenges, but rely on the strict assumption that the corruption admits a convolutional form. Similarly, specialized solutions in imaging~\citep{levin2009understanding, katz2014non, bertolotti2012non}, seismology~\citep{kaaresen1998multichannel}, and neuroscience~\citep{makeig1996independent} achieve impressive robustness and computational efficiency through application-specific physical models. In contrast, our approach targets a more agnostic setting, where explicit physical priors, convolutional forms, and optimized domain-specific pipelines are unavailable. We instead establish symmetry and locality as general principles for recovering latent structure from unstructured observations.

\textbf{Symmetry discovery.}
Existing symmetry discovery methods learn equivariances in supervised or structured prediction settings~\citep{romero2021learning, santosescriche2025, forestano2023deep, moskalev2022liegg, bhat2025atlasd, zhou2020meta}, uncover dynamical mechanisms~\citep{greydanus2019hamiltonian, alet2021noether, Churchill_2023, sohl2010unsupervised, Champion_2019, dehmamy2021automatic, koyama2023neural, miyato2022unsupervised, park2022learning, yang2023latent, mitchel2024neural}, or identify invariant operators~\citep{benton2020learning, desai2022symmetry, tombs2022method, yang2023generative, efe2025, shaw2024symmetrydiscoveryaffinetransformations}. These approaches typically assume that the relevant domain topology or sample coordinates are already specified, or that the symmetry admits a representation in the original observation space. In contrast, we define symmetries in an embedding space and use symmetry discovery as a mechanism for recovering latent domains and signals under unknown measurement operators.

\textbf{Graph and manifold learning.}
Graph and manifold learning methods infer latent structures by assuming fixed nodes~\citep{franceschi2019learning, zhu2021survey} or unstructured point clouds~\citep{tenenbaum2000global, mcinnes2020umap}. These well-established methods correspond to different generative assumptions over the data. Rather than competing with these highly optimized tools in their native settings, we propose a complementary latent-signal perspective based on symmetry and locality priors. This positions our framework at the intersection of blind inverse problems, symmetry learning, and structure recovery.
\section{Problem setting}\label{sec:problem-setting}
\definecolor{latentred}{RGB}{210,45,45}
\definecolor{latentorange}{RGB}{230,135,35}
\definecolor{latentblue}{RGB}{55,95,180}
\definecolor{softgray}{RGB}{245,246,248}

\newcommand{\pointsource}[5]{%
    \begin{scope}[shift={(#1,#2)}]
        \fill[#3, opacity=0.08] (0,0) ellipse (0.52 and 0.24);
        \fill[#3, opacity=0.16] (0,0) ellipse (0.34 and 0.16);
        \fill[#3, opacity=0.24] (0,0) ellipse (0.20 and 0.10);
        \draw[#3!85!black, line width=0.9pt] (0,0) -- (0,#4);
        \fill[#3!85!black] (0,#4) circle (1.2pt);
        \node[font=\scriptsize, text=#3!75!black, anchor=south] at (0,#4+0.06) {#5};
    \end{scope}
}

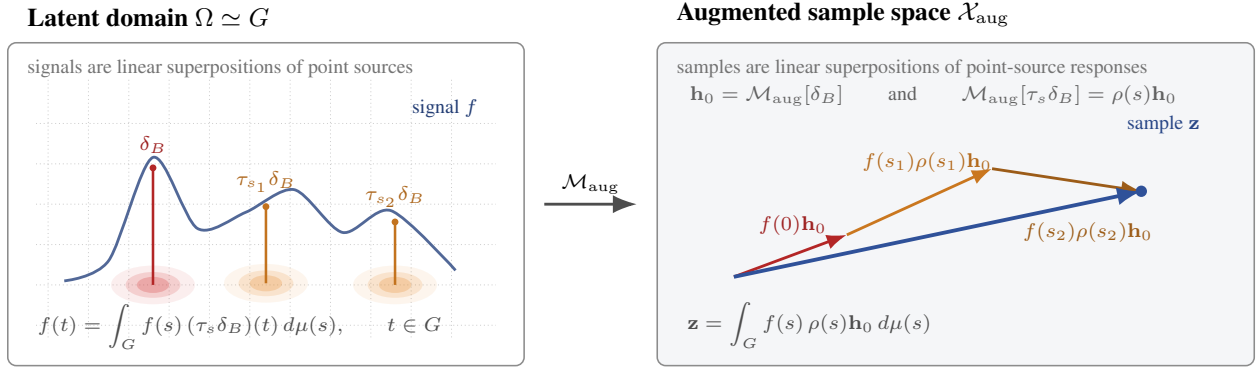
\begin{figure}[t]
\centering
\resizebox{\linewidth}{!}{%
\begin{tikzpicture}[
    >=Latex,
    font=\small,
    every node/.style={align=center},
    panel/.style={rounded corners=3pt, draw=black!45, fill=softgray, line width=0.55pt},
    maparrow/.style={-{Latex[length=3mm]}, line width=1.0pt, draw=black!70},
    actionarrow/.style={-{Latex[length=2.2mm]}, line width=0.9pt, draw=latentred!80!black},
    vecarrow/.style={-{Latex[length=2.5mm]}, line width=1.0pt}
]

\begin{scope}[shift={(0,0)}]
    \filldraw[panel, fill=white] (0,0) rectangle (6.4,4.0);

    \node[font=\bfseries\small, anchor=south west] at (0.12,4.08)
        {Latent domain $\Omega \simeq G$};
    \node[font=\scriptsize, text=black!55, anchor=north west] at (0.12,3.92)
        {signals are linear superpositions of point sources};

    \foreach \x in {0.5,1.0,...,5.9}{
        \draw[black!20, densely dotted, line width=0.35pt] (\x,0.35) -- (\x,3.55);
    }
    \foreach \y in {0.5,1.0,...,3.2}{
        \draw[black!20, densely dotted, line width=0.35pt] (0.35,\y) -- (6.05,\y);
    }

    \draw[latentblue!70!black, line width=1.0pt, opacity=0.8]
        plot[smooth] coordinates {
            (0.70,1.05) (1.25,1.30) (1.80,2.58)
            (2.35,1.70) (2.95,1.90) (3.55,2.18)
            (4.15,1.65) (4.75,1.92) (5.55,1.18)
        };
    \node[font=\scriptsize, text=latentblue!65!black, anchor=south east]
        at (5.95,2.95) {signal $f$};

    \pointsource{1.8}{1.00}{latentred}{1.45}{$\delta_B$}
    \pointsource{3.2}{1.02}{latentorange}{0.95}{$\tau_{s_1}\delta_B$}
    \pointsource{4.8}{1.00}{latentorange}{0.78}{$\tau_{s_2}\delta_B$}

    \node[font=\scriptsize, text=black!70, anchor=south west] at (0.25,0.08)
        {$f(t)=\displaystyle\int_G f(s)\,(\tau_s\delta_B)(t)\,d\mu(s),\qquad t\in G$};
\end{scope}

\draw[maparrow] (6.65,2.0) -- node[above, font=\scriptsize] {$\mathcal M_{\mathrm{aug}}$} (7.80,2.0);

\begin{scope}[shift={(8.05,0)}]
    \filldraw[panel] (0,0) rectangle (7.4,4.0);

    \node[font=\bfseries\small, anchor=south west] at (0.12,4.08)
        {Augmented sample space $\mathcal X_{\mathrm{aug}}$};
    \node[font=\scriptsize, text=black!55, anchor=north west] at (0.12,3.92)
        {samples are linear superpositions of point-source responses};

    \node[font=\scriptsize, text=black!62, anchor=west] at (0.28,3.35)
        {$\ve h_0=\mathcal M_{\mathrm{aug}}[\delta_B]$ \qquad and \qquad
         $\mathcal M_{\mathrm{aug}}[\tau_s\delta_B]=\rho(s)\ve h_0$};

    \coordinate (O)  at (0.95,1.10);
    \coordinate (A)  at ($(O)+(1.40,0.52)$);
    \coordinate (B)  at ($(A)+(1.80,0.82)$);
    \coordinate (C)  at ($(B)+(1.85,-0.28)$);

    \draw[vecarrow, draw=latentred!85!black] (O) -- (A);
    \node[font=\scriptsize, text=latentred!75!black, anchor=south]
        at ($(O)!0.52!(A)+(0.00,0.12)$)
        {$f(0)\ve h_0$};

    \draw[vecarrow, draw=latentorange!88!black] (A) -- (B);
        \node[font=\scriptsize, text=latentorange!78!black, anchor=south]
        at ($(A)!0.52!(B)+(0.06,0.23)$)
        {$f(s_1)\rho(s_1)\ve h_0$};

    \draw[vecarrow, draw=latentorange!68!black] (B) -- (C);
    \node[font=\scriptsize, text=latentorange!68!black, anchor=north]
        at ($(B)!0.48!(C)+(0.30,-0.4)$)
        {$f(s_2)\rho(s_2)\ve h_0$};

    \draw[-{Latex[length=3mm]}, draw=latentblue!85!black, line width=1.35pt] (O) -- (C);
    \fill[latentblue!85!black] (C) circle (2.0pt);

    \node[font=\scriptsize, text=latentblue!75!black, anchor=south east]
        at (6.85,2.75) {sample $\ve z$};

    \node[font=\scriptsize, text=black!70, anchor=south west] at (0.25,0.10)
        {$\ve z=\displaystyle\int_G f(s)\,\rho(s)\ve h_0\,d\mu(s)$};
\end{scope}

\end{tikzpicture}%
}
\caption{\textbf{Symmetry-based domain recovery.}
Latent signals are superpositions of translated band-limited point sources. Under the augmented measurement, these become orbit elements \(\rho(s)\ve h_0\), whose superpositions form samples \(\ve z\). Learning this symmetry orbit reveals the latent domain structure.}
\label{fig:two-box-domain-recovery}
\end{figure}

\paragraph{Using a translation symmetry group as a domain surrogate.}
Our formulation accesses the latent domain through translation symmetry. We assume that the domain \(\Omega\) is a homogeneous space equipped with the \emph{regular action} of a locally compact Abelian Lie group \(G\). Here, \(G\) generalizes translations, while regularity gives a one-to-one correspondence between domain points and group elements. Therefore, after fixing a point \(p_0\in\Omega\), every point \(p\in\Omega\) can be written uniquely as \(p=t\cdot p_0\). Thus, instead of defining latent signals over domain coordinates, we define them over the group as square-integrable functions \(f\in L^2(G,\mathbb R)\).

\paragraph{Measurement operator.}\label{subsec:measurement-operator}
We assume that the observed dataset \(\mathcal D=\{\ve{x}_i\}_{i=1}^N\), with \(\ve{x}_i\in\mathcal X=\mathbb R^d\), is generated by applying a linear measurement operator \(\mathcal M\) to latent signals \(f_i\in\mathcal F_B\). Here, \(\mathcal F_B\subset L^2(G,\mathbb R)\) denotes a band-limited subspace, allowing continuous signals to be represented as discrete vectors through a bijection.  

For symmetries that are not well represented by an orthogonal action on \(\mathcal X\), we introduce a linear embedding \(\mathbf E:\mathcal X\to\mathcal X_{\mathrm{aug}}\), where \(\mathcal X_{\mathrm{aug}}=\mathbb R^{d_{\mathrm{aug}}}\). We define the \textbf{augmented measurement operator} as \(\mathcal M_{\mathrm{aug}}:=\mathbf E\mathcal M\), and the \textbf{augmented sample} as \(\ve z=\mathbf E\ve x\).

\paragraph{Induced representation.}\label{subsec:induced-representation}
Our framework is built upon a function \(\delta_B \in \mathcal F_B\), which we call a point source. The point source \(\delta_B\) is the band-limited analogue of a Dirac delta, namely the reproducing kernel of \(\mathcal F_B\). We assume that translating the point source under the group action \(\tau\) induces a special orthogonal representation \(\rho:G\to\mathrm{SO}(d_{\mathrm{aug}})\) on \(\mathcal X_{\mathrm{aug}}\), satisfying
\begin{equation}\label{eq:intertwine}
\mathcal{M}_{\mathrm{aug}}[\tau_t \delta_B] = \rho(t) \ve{h}_0\,,
\end{equation}
where we refer \(\ve h_0:=\mathcal M_{\mathrm{aug}}[\delta_B]\) as \textbf{point-source response}. To make \(\rho\) well defined on the orbit, we require \(\mathcal M_{\mathrm{aug}}\) to be injective on \(\mathcal O_{\delta_B}=\{\tau_t\delta_B\mid t\in G\}\).

\paragraph{Data-generating process.}
Any signal in the latent domain can be written as a superposition of translated point sources. Thus, writing \((\tau_s\delta_B)(t)\) for the band-limited point source centered at \(s\), for any \(f\in\mathcal F_B\),
\begin{equation}\label{eq:kernel-identity}
    f(t)=\int_G f(s)(\tau_s \delta_B)(t)\,d\mu(s),
\end{equation}
where \(\mu\) is the Haar measure on \(G\). Applying \(\mathcal M_{\mathrm{aug}}\) and using linearity gives
\begin{equation}\label{eq:forward-model}
    \ve{z}
    =
    \int_G f(s)\rho(s)\ve{h}_0\,d\mu(s).
\end{equation}
Thus, each augmented sample is a weighted superposition of orbit elements \(\rho(s)\ve h_0\), which span the \textbf{observable sample space}
\[
\mathcal X_{\mathrm{obs}}
:=
\operatorname{span}\{\rho(s)\ve h_0:s\in G\}
\subseteq \mathcal X_{\mathrm{aug}}.
\]

\paragraph{Homogeneous case.}\label{problem-setting:homogeneous-case}
In some experimental settings, we encounter a case where \(f(t)\) is not a general superposition of orbit elements in \(\mathcal O_{\delta_B}\), but a single orbit element. We refer to this as the homogeneous case, which corresponds to \(f(t)=(\tau_{t_0}\delta_B)(t)\), where \(t_0 \in G\) identifies the specific orbit element. Then \(\ve z=\mathcal M_{\mathrm{aug}}[\tau_{t_0}\delta_B]=\rho(t_0)\ve h_0\), so homogeneous datasets lie on the observed orbit \(\mathcal O_{\ve h_0}=\{\rho(t_0)\ve h_0:t_0\in G\}\), whereas the general model also permits superpositions of orbit elements. 

\paragraph{Labeling latent and learned symmetry actions.}
We parameterize both latent and learned symmetry representations by \(P\) commuting skew-symmetric generators, \(\mathcal L=(\mathbf L_1,\dots,\mathbf L_P)\) and \(\hat{\mathcal L}=(\hat{\mathbf L}_1,\dots,\hat{\mathbf L}_P)\). For coordinates \(\ve t=(t_1,\dots,t_P)\) on the Abelian Lie group, define
\(
\mathbf R(\ve t;\mathcal L)
:=
\exp\!\left(\sum_{j=1}^P t_j\mathbf L_j\right).
\)
Then \(\rho(\ve t)=\mathbf R(\ve t;\mathcal L)\) and \(\hat\rho(\ve t)=\mathbf R(\ve t;\hat{\mathcal L})\).
\section{Methodology}\label{sec:methodology}

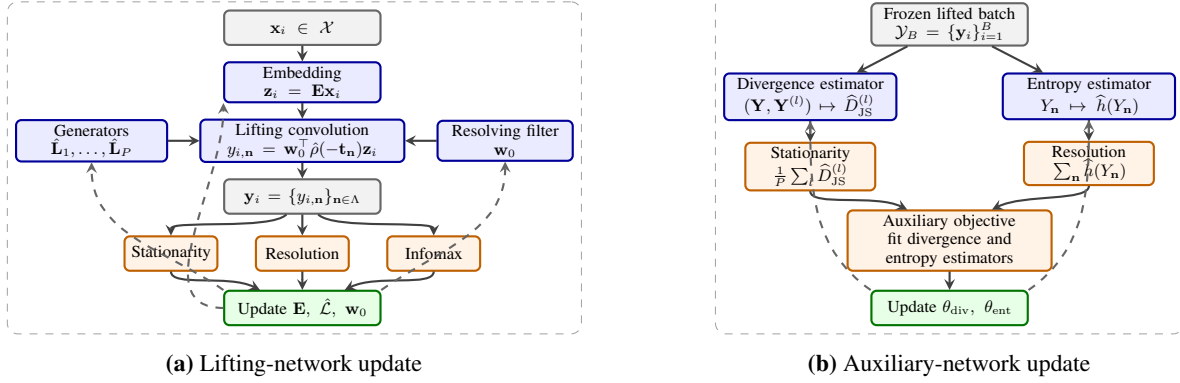
\begin{figure*}[t]
\centering


\tikzset{
  flow/.style={
    ->,
    thick,
    draw=black!75
  },
  updatearrow/.style={
    ->,
    thick,
    dashed,
    draw=black!60
  },
  data/.style={
    rectangle,
    rounded corners=2pt,
    draw=black!60,
    thick,
    fill=gray!10,
    align=center,
    text width=28mm,
    minimum height=6.5mm,
    inner sep=2.2pt,
    font=\small
  },
  model/.style={
    rectangle,
    rounded corners=2pt,
    draw=blue!60!black,
    thick,
    fill=blue!8,
    align=center,
    text width=28mm,
    minimum height=6.5mm,
    inner sep=2.2pt,
    font=\small
  },
  modelwide/.style={
    rectangle,
    rounded corners=2pt,
    draw=blue!60!black,
    thick,
    fill=blue!8,
    align=center,
    text width=37mm,
    minimum height=6.5mm,
    inner sep=2.2pt,
    font=\small
  },
  side/.style={
    rectangle,
    rounded corners=2pt,
    draw=blue!60!black,
    thick,
    fill=blue!8,
    align=center,
    text width=24mm,
    minimum height=6.5mm,
    inner sep=2.2pt,
    font=\small
  },
  sidewide/.style={
    rectangle,
    rounded corners=2pt,
    draw=blue!60!black,
    thick,
    fill=blue!8,
    align=center,
    text width=27mm,
    minimum height=6.5mm,
    inner sep=2.2pt,
    font=\small
  },
  loss/.style={
    rectangle,
    rounded corners=2pt,
    draw=orange!75!black,
    thick,
    fill=orange!10,
    align=center,
    text width=16mm,
    minimum height=6.5mm,
    inner sep=2.2pt,
    font=\small
  },
  losswide/.style={
    rectangle,
    rounded corners=2pt,
    draw=orange!75!black,
    thick,
    fill=orange!10,
    align=center,
    text width=26mm,
    minimum height=6.5mm,
    inner sep=2.2pt,
    font=\small
  },
  objective/.style={
    rectangle,
    rounded corners=2pt,
    draw=orange!75!black,
    thick,
    fill=orange!10,
    align=center,
    text width=37mm,
    minimum height=6.5mm,
    inner sep=2.2pt,
    font=\small
  },
  updatebox/.style={
    rectangle,
    rounded corners=2pt,
    draw=green!45!black,
    thick,
    fill=green!10,
    align=center,
    text width=28mm,
    minimum height=6.5mm,
    inner sep=2.2pt,
    font=\small
  }
}

\begin{minipage}[t]{0.475\textwidth}
\centering
\begin{tikzpicture}[
    node distance=3mm and 6mm,
    >=stealth,
    scale=0.7,
    every node/.style={transform shape}
]

\node[data] (x) {$\ve{x}_i\in\mathcal{X}$};

\node[model, below=of x] (embed)
{Embedding\\
$\ve{z}_i = \mathbf E\ve{x}_i$};

\node[modelwide, below=of embed] (lift)
{Lifting convolution\\
$y_{i,\ve n}
=
\ve w_0^\top
\hat{\rho}(-\ve t_{\ve n})
\ve z_i$};

\node[data, below=of lift] (y)
{$\mathbf y_i=\{y_{i,\ve n}\}_{\ve n\in\Lambda}$};

\node[sidewide, left=6mm of lift] (gens)
{Generators\\
$\hat{\mathbf L}_1,\dots,\hat{\mathbf L}_P$};

\node[side, right=6mm of lift] (phi)
{Resolving filter\\
$\ve w_0$};

\node[loss, below left=4mm and 1mm of y] (stat)
{Stationarity};

\node[loss, below=4mm of y] (res)
{Resolution};

\node[loss, below right=4mm and 1mm of y] (info)
{Infomax};

\node[updatebox, below=3.5mm of res] (upd)
{Update $\mathbf E,\;\hat{\mathcal L},\;\ve w_0$};

\draw[flow] (x) -- (embed);
\draw[flow] (embed) -- (lift);
\draw[flow] (lift) -- (y);

\draw[flow] (gens) -- (lift);
\draw[flow] (phi) -- (lift);

\draw[flow] ([xshift=-3mm]y.south) .. controls +(-3mm,-3mm) and +(0,3mm) .. (stat.north);
\draw[flow] (y.south) -- (res.north);
\draw[flow] ([xshift=3mm]y.south) .. controls +(3mm,-3mm) and +(0,3mm) .. (info.north);

\draw[flow] (stat.south) .. controls +(0,-3mm) and +(-6mm,3mm) .. ([xshift=-8mm]upd.north);
\draw[flow] (res.south)  .. controls +(0,-3mm) and +(0,3mm)    .. (upd.north);
\draw[flow] (info.south) .. controls +(0,-3mm) and +(6mm,3mm)  .. ([xshift=8mm]upd.north);

\draw[updatearrow] (upd.west) to[out=180,in=-90] (embed.south west);
\draw[updatearrow] (upd.north west) to[out=145,in=-90] (gens.south);
\draw[updatearrow] (upd.north east) to[out=35,in=-90] (phi.south);

\node[
  draw=black!35,
  dashed,
  rounded corners=4pt,
  fit=(x)(embed)(lift)(y)(gens)(phi)(stat)(res)(info)(upd),
  inner sep=3pt
] {};

\end{tikzpicture}

\vspace{1mm}
{\small\textbf{(a)} Lifting-network update}
\end{minipage}
\hfill
\begin{minipage}[t]{0.475\textwidth}
\centering
\begin{tikzpicture}[
    node distance=3mm and 6mm,
    >=stealth,
    scale=0.7,
    every node/.style={transform shape}
]

\tikzset{
  auxmodel/.style={
    rectangle,
    rounded corners=2pt,
    draw=blue!60!black,
    thick,
    fill=blue!8,
    align=center,
    text width=31mm,
    minimum height=6.5mm,
    inner sep=2.2pt,
    font=\small
  },
  auxloss/.style={
    rectangle,
    rounded corners=2pt,
    draw=orange!75!black,
    thick,
    fill=orange!10,
    align=center,
    text width=23mm,
    minimum height=6.5mm,
    inner sep=2.2pt,
    font=\small
  }
}

\node[data] (yin)
{Frozen lifted batch\\
$\mathcal Y_B=\{\mathbf y_i\}_{i=1}^B$};

\node[auxmodel, below left=5mm and -5mm of yin] (divest)
{Divergence estimator\\
$(\mathbf Y,\mathbf Y^{(l)}) \mapsto \widehat D_{\mathrm{JS}}^{(l)}$};

\node[auxmodel, below right=5mm and -5mm of yin] (entest)
{Entropy estimator\\
$Y_{\ve n} \mapsto \widehat h(Y_{\ve n})$};

\node[auxloss, below=4mm of divest] (stat2)
{Stationarity\\
$\frac{1}{P}\sum_l \widehat D_{\mathrm{JS}}^{(l)}$};

\node[auxloss, below=4mm of entest] (res2)
{Resolution\\
$\sum_{\ve n}\widehat h(Y_{\ve n})$};

\path (stat2.south) -- (res2.south) coordinate[midway] (lossmid);

\node[objective, below=4mm of lossmid] (obj2)
{Auxiliary objective\\
fit divergence and entropy estimators};

\node[updatebox, below=3.5mm of obj2] (upd2)
{Update $\theta_{\mathrm{div}},\;\theta_{\mathrm{ent}}$};

\draw[flow] (yin) -- (divest);
\draw[flow] (yin) -- (entest);

\draw[flow] (divest) -- (stat2);
\draw[flow] (entest) -- (res2);

\draw[flow] (stat2.south) .. controls +(0,-3mm) and +(-6mm,3mm) .. ([xshift=-8mm]obj2.north);
\draw[flow] (res2.south)  .. controls +(0,-3mm) and +(6mm,3mm)  .. ([xshift=8mm]obj2.north);

\draw[flow] (obj2) -- (upd2);

\draw[updatearrow] (upd2.north west) to[out=145,in=-90] (divest.south);
\draw[updatearrow] (upd2.north east) to[out=35,in=-90] (entest.south);

\node[
  draw=black!35,
  dashed,
  rounded corners=4pt,
  fit=(yin)(divest)(entest)(stat2)(res2)(obj2)(upd2),
  inner sep=3pt
] {};

\end{tikzpicture}

\vspace{1mm}
{\small\textbf{(b)} Auxiliary-network update}
\end{minipage}

\caption{
\textbf{Training scheme.}
\textbf{(a)} The lifting network update optimizes \(\mathbf E\), \(\hat{\mathcal L}\), and \(\ve w_0\) using stationarity, resolution, and InfoMax losses.
\textbf{(b)} The auxiliary network update fits statistical divergence and entropy estimators.
}
\label{fig:alternating-training-loops}
\end{figure*}

\subsection{Architecture}
We formulate latent domain recovery as learning a symmetry-based representation that satisfies stationarity, resolution, and information-maximization priors. The framework consists of a \emph{lifting network}, which maps each unstructured sample \(\ve{x}_i\in\mathcal X\) to a group-indexed representation, and an \emph{auxiliary network} for entropy and statistical divergence terms. Figure~\ref{fig:alternating-training-loops} visualizes the training loops for the lifting and auxiliary networks.

\paragraph{Lifting network.}
Inspired by lifting layers in group convolutional networks~\cite{bekkers2024fastexpressivesenequivariant}, we define a learnable lifting operator \(\mathcal U:\mathcal X_{\mathrm{aug}} \to L^2(G)\) by
\begin{equation}\label{eq:lifting-continuous}
    y_i(\ve t) = (\mathcal U \ve{z}_i)(\ve t) := \ve{w}_0^\top\hat\rho(-\ve t)\ve{z}_i ,
\end{equation}
where \(\ve{z}_i =\mathbf E \ve{x}_i\), \(\ve{w}_0 \in\mathcal X_{\mathrm{aug}}\) is the learned resolving filter, and \(\hat\rho(\ve t)=\mathbf{R}(\ve t;\hat{\mathcal L})\) is parameterized by commuting skew-symmetric generators 
\(\hat{\mathcal L} = (\hat{\mathbf{L}}_1, \dots, \hat{\mathbf{L}}_P)\).

Sampling on a regular grid \(\{\ve t_{\ve n}\}_{\ve n\in\Lambda}\), with \(\Lambda\subset\mathbb Z^P\), yields
\[
\mathbf y_i=\{y_{i,\ve n}\}_{\ve n\in\Lambda},
    \qquad
    y_{i,\ve n}
    =
    \ve{w}_0^\top
    \hat\rho(-\ve t_{\ve n})
    \mathbf E \ve{x}_i .
\]
For a mini-batch, we write \(\mathcal Y_B=\{\mathbf y_i\}_{i=1}^B\). The trainable parameters are \(\mathbf E\), \(\hat{\mathcal L}\), and \(\ve w_0\). Since the generators commute, lifting can be implemented efficiently by simultaneous diagonalization, as detailed in Appendix~\ref{app:efficient-lifting-implementation}.

In Appendix~\ref{prop:oracle-deconvolution}, we show that this form admits exact recovery in the torus setting, provided that the point-source response is nonzero on all bandlimited modes. Equivalently, \(f(\ve t)=\ve w_\star^\dagger \rho(-\ve t)\ve z\) for some resolving filter \(\ve w_\star\).

\paragraph{Auxiliary estimators and alternating optimization.}
The loss terms require entropy and divergence estimates from finite batches. We use trainable auxiliary modules on \(\mathcal Y_B\), and estimate \(h(\mathbf Y)\) from the batch covariance using an adaptive-rank Gaussian approximation. Training alternates between updating the auxiliary modules with fixed lifting parameters and updating the lifting network with fixed estimators.

\subsection{Learning objective}
We train the lifting network with three objectives: stationarity, resolution, and InfoMax:
\begin{equation}
\label{eq:loss-lifting}
\mathcal{J} = \mathcal{J}_{\mathrm{stationarity}} + \alpha \mathcal{J}_{\mathrm{resolution}} + \beta \mathcal{J}_{\mathrm{infomax}} .
\end{equation}
In Section~\ref{sec:identifiability}, we show how these loss terms work cooperatively to discover the latent symmetry action and reveal latent signals.

\textbf{Stationarity.} Let \(\mathbf Y^{(l)}\) be \(\mathbf Y\) shifted by one grid step along axis \(l\), so \(Y^{(l)}_{\ve n}=Y_{\ve n+\ve e_l}\). We enforce stationarity by penalizing
\(
    \mathcal{J}_{\mathrm{stationarity}}
    :=
    \frac{1}{P}
    \sum_{l=1}^P
    D_{\mathrm{JS}}(P_{\mathbf Y} \| P_{\mathbf Y^{(l)}}).
\)

\textbf{Resolution.} Recovering the symmetry action alone does not determine the resolving filter \(\ve{w}_0\). We minimize total correlation,
\(
    \mathcal{J}_{\mathrm{resolution}}
    :=
    \sum_{\ve n\in\Lambda} h(Y_{\ve n}) - h(\mathbf Y),
\)
to align the model output with the latent locally correlated stochastic field.

\textbf{Information maximization.} To prevent collapse, we maximize joint Gaussian entropy by minimizing
\(
    \mathcal{J}_{\mathrm{infomax}} := -h(\mathbf Y),
\)
estimated using an adaptive-rank approximation, as detailed in Appendix~\ref{subsection:auxiliary-network}.

\subsection{Implementation details}
The main computational costs are eigendecompositions for adaptive-rank entropy estimation and efficient exponentiation of commuting generators. Full complexity analysis, hyperparameters, and training schedules are given in Appendices~\ref{app:training-details} and~\ref{sec:morris_analysis}.
\section{Identifiability and theoretical interpretation}\label{sec:identifiability}

In this section, we argue that the proposed loss terms encourage the discovery of the latent symmetry action and the recovery of latent signals, assuming that the latent signal can be modeled as a locally correlated stationary stochastic field. We start by investigating the forward model proposed in Section~\ref{sec:problem-setting}:
\begin{equation}
    y(\ve t)
    =
    \ve{w}_0^\top \hat\rho(-\ve t)\ve z
    =
    \int_G f(\ve s)\, k(\ve s,\ve t)\,d\mu(\ve s),
    \qquad
    k(\ve s,\ve t)
    :=
    \ve{w}_0^\top \hat\rho(-\ve t)\rho(\ve s)\ve{h}_0 .
\end{equation}
Thus, lifting combined with the forward model leads to an operator that we refer to as the \textbf{recovery operator} \(\mathcal K:\mathcal F_B\to L^2(G,\mathbb R)\), with \(y=\mathcal K f\). We show that stationarity encourages the recovery operator to have a convolutional structure, which also leads to symmetry discovery, while the InfoMax and resolution objectives encourage the convolutional kernel toward a Dirac delta, enabling signal recovery.

Throughout the derivations, we use uppercase letters for random fields and lowercase letters for signals, which are their realizations. Translations act on realizations by \((T_{\ve{\alpha}}h)(\ve t)=h(\ve t-\ve{\alpha})\).

\subsection{Symmetry discovery}

\paragraph{From stationarity to equivariance.}
The latent field is stationary by assumption, and the stationarity loss encourages the output \(Y\) to be stationary, implying translation equivariance for 
\(\mathcal K\) under a symmetry-uniqueness assumption on the lifted space.

\begin{proposition}[Translation symmetries imply equivariance]
Let \(F\) be stationary and let \(Y=\mathcal K F\). Assume that the recovery operator \(\mathcal K\) is injective on the relevant signal subspace, and that the only distributional symmetries of \(Y\) are translations in the lifted coordinates. Then \(\mathcal K\) is translation-equivariant up to a group reparameterization,
\(
    \mathcal K T_{\ve{\alpha}} = T^{(Y)}_{\ve{\beta}(\ve{\alpha})}\mathcal K .
\)
\end{proposition}

\noindent
The proof, given in Appendix~\ref{proof:translation-only-symmetries-imply-equivariance}, follows by matching latent and lifted translation actions.

\paragraph{From equivariance to convolution.}
Translation equivariance restricts the recovery operator \(\mathcal K\) to a convolutional form.

\begin{proposition}[Equivariant recovery yields a shift-invariant kernel]
\label{prop:recovery-kernel-shift-invariance}
Assume the recovery operator \(\mathcal K\) satisfies
\(
    \mathcal K T_{\ve{\alpha}} = T^{(Y)}_{\ve{\beta}(\ve{\alpha})} \mathcal K
\)
for an invertible group reparameterization \(\ve{\beta}: \mathbb{R}^P \to \mathbb{R}^P\). Then
\(
    k(\ve{s},\ve{t}) = \kappa(\ve{s}-\ve{\beta}^{-1}(\ve{t}))
\)
for some function \(\kappa: \mathbb{R}^P \to \mathbb{R}\).
\end{proposition}

\noindent

The proof is given in Appendix~\ref{proof:recovery-kernel-shift-invariance}. Although \(\ve{\beta}(\ve{\alpha})=\mathbf A\ve{\alpha}\) for some \(\mathbf A\in \mathrm{GL}(P)\) in general, the InfoMax loss restricts \(\mathbf A\) to coordinate permutations and direction flips on compact rectangular domains. Thus, up to these ambiguities, we write \(k(\ve{s},\ve{t}) = \kappa(\ve{s}-\ve{t})\). 

\paragraph{From convolution to symmetry identification.}
The InfoMax objective encourages the learned resolving filter to span the observable sample space \(\mathcal X_{\mathrm{obs}}\). Combined with the shift-invariance of the kernel \(k\), this aligns the latent and learned symmetry representations. 

\begin{proposition}[Symmetry identification]
\label{prop:symmetry_discovery}
Assume that
\(
    k(\ve{s},\ve{t})
    =
    \ve{w}_0^\top\hat\rho(-\ve{t})\rho(\ve{s})\ve{h}_0
\)
is shift-invariant, and that the filter orbit spans the observable sample space,
\(
    \operatorname{span}\{\hat\rho(\ve{t})\ve{w}_0:\ve{t}\in \mathbb{R}^P\}
    =
    \mathcal X_{\mathrm{obs}} .
\)
Then, for every \(\ve{r} \in \mathbb{R}^P\),
\[
    \hat\rho(\ve{r})=\rho(\ve{r}) \qquad \text{on } \mathcal X_{\mathrm{obs}} .
\]
\end{proposition}

\noindent
The proof, given in Appendix~\ref{proof:symmetry_discovery}, analyzes the form $\ve{w}_0^\top\hat\rho(-\ve{t}) \bigl(\rho(\ve{r})-\hat\rho(\ve{r})\bigr)\rho(\ve{s})\ve{h}_0 = 0$ for any $\ve{s}, \ve{t} \in \mathbb{R}^P$.

\subsection{Signal recovery}
Once the symmetry is identified, signal recovery reduces to learning a delta-like convolutional kernel. We analyze this on a discretized 1D torus with \(P=1\) and \(N\) latent points.

\paragraph{InfoMax flattens the kernel spectrum.}
The convolutional kernel has a fixed \(L^2\) norm, as proved in Appendix~\ref{lem:recovery-kernel-has-fixed-energy}. Under this constraint, InfoMax spreads kernel energy uniformly across Fourier modes.

\begin{lemma}[Maximizing Gaussian entropy yields a flat spectrum]\label{lem:infomax-implies-flat-kernel-spectrum}
Let \(\tilde{\kappa}_m\) for \(m \in \{0, \dots, N-1\}\) be the Fourier coefficients of the kernel. Assume the latent field is stationary with spectral power \(s_m > 0\), and the kernel \(\kappa\) has fixed \(L^2\) norm \(C > 0\). Then maximizing the Gaussian entropy yields a flat spectrum,
\(
    |\tilde{\kappa}_m|^2 = \frac{C}{N}.
\)
\end{lemma}

\noindent
The proof, given in Appendix~\ref{proof:infomax-implies-flat-kernel-spectrum}, leverages Fourier diagonalization and entropy maximization.

\paragraph{Flatness and compactness force a shifted delta.}
Appendix~\ref{analysis-of-resolution} shows that total-correlation minimization favors compact kernels: broad kernels introduce dependencies. Under this compactness bias, flat spectrum forces a shifted delta.

\begin{proposition}\label{prop:compact-support-flat-filter-implies-delta}
Assume that the nonzero kernel \(\kappa_n\) has a flat spectrum and minimal contiguous support of length at most \(N/2\). Then \(\kappa\) is a shifted and scaled Kronecker delta.
\end{proposition}

\noindent
The proof, given in Appendix~\ref{proof:compact-support-flat-filter-implies-delta}, uses the autocorrelation of flat-spectrum filters.
\section{Experiments}\label{sec:experiments}
We evaluate latent domain recovery and symmetry discovery on synthetic waveforms, images, and black-box neural spike-rate vectors. Establishing benchmarks requires additional layers to adapt existing methods in most cases; we detail these in the relevant experiments. 

We compare our method directly with symmetry discovery algorithms such as LieGAN~\cite{yang2023latent}. For domain recovery, we adapt graph structure learning algorithms (e.g., GLASSO and Covariance Matching) and manifold discovery algorithms (e.g., Isomap and UMAP). We also tested Topographic ICA (TICA) but consistently observed convergence failures. 

We use fixed hyperparameters for our method, whereas baseline methods are tuned via grid search to select the best configuration. Baseline details, negative results, and hyperparameter grids are reported in Appendix~\ref{baselines}.

\textbf{Evaluation protocol.}
For evaluating symmetry discovery capability, we compare learned effective generators $L_{i,\mathrm{eff}}=E^\dagger L_iE$ with ground-truth finite-difference generators $\Delta_i$. This comparison is performed after transporting learned actions to native coordinates via $A^{-1}L_{i,\mathrm{eff}}A$, where $A$ is the obfuscating transform. Since finite differences are dispersive at high frequencies, we report the Frobenius cosine similarity after projection to a low-frequency Fourier subspace with a normalized bandwidth of $\beta=0.75$, denoted as $S_{0.75}$. For domain recovery, we report the maximum absolute Pearson correlation between the latent signal and the model output over all admissible axis permutations and reflections (i.e., $2^P P!$ alignments for a $P$-dimensional translation group).

\subsection{Waveform benchmarks.}
\textbf{Comparison with symmetry-discovery baseline.}
We compare translation-symmetry discovery against LieGAN~\cite{yang2023latent} on low-dimensional Generalized Shot Noise (GSN) signals~\cite{rice1944mathematical} with Gaussian bases ($d=15$). We use no embedding layer ($\mathcal X=\mathcal X_A$) to match LieGAN's coordinate-action assumptions. Both models use identical datasets and training budgets.

We evaluate FFT-based fractional cyclic shifts, continuous latent shifts before discretization, and integer-only discrete shifts. LieGAN performs strongly for FFT-based shifts, which define exact cyclic actions on the discretized signal. However, it degrades under partial observation or discretization noise. Our method remains robust, suggesting that its information-theoretic objectives provide more stable recovery signals (Table~\ref{tab:waveform_mnist_benchmarks}).

\textbf{Comparison with domain-recovery baselines.}
We test latent domain recovery on GSN processes~\cite{rice1944mathematical} with Gaussian and Legendre bases, and on Ising spin chains~\cite{kulske2025ising} under different linear transformations (Appendix~\ref{appsubsec:gsn-process}, \ref{appsubsec:ising}). Our method recovers both translation generators and point-response filters. In GSN experiments, these transform consistently under DST-I (Figure~\ref{fig:waveform_experiments}a,b), allowing the learned lifting to approximately invert the transform (Figure~\ref{fig:waveform_experiments}c). In the blind Ising setting, the method similarly recovers the latent signal structure and transformed generator after a random invertible transform.

For benchmarking, we apply two graph structure learning algorithms: GLASSO and Covariance Matching (CM). The GLASSO benchmark involves estimating the graph precision matrix and subsequently aligning it with a reference 1D grid adjacency matrix. For CM, we align the empirical covariance matrix eigenvectors with those of the 1D grid adjacency matrix. Additionally, we use a covariance eigenvalue decay prior to impose compactness on the revealed signals.

GLASSO is competitive for native Gaussian GSN signals, but drops under DST-I from $r=0.812$ to $r=0.352$. This is likely because the transformation violates its sparse precision-matrix assumption. CM is less basis-sensitive but remains below $r=0.6$, despite its smoothness prior. Across settings, our method achieves stronger signal recovery with $r > 0.85$ while simultaneously identifying the associated translation generator (Table~\ref{tab:waveform_mnist_benchmarks}).

\begin{figure}[t]
  \centering
  \includegraphics[
    trim={0cm 0cm 0cm 0cm},
    clip,
    width=0.9\linewidth
  ]{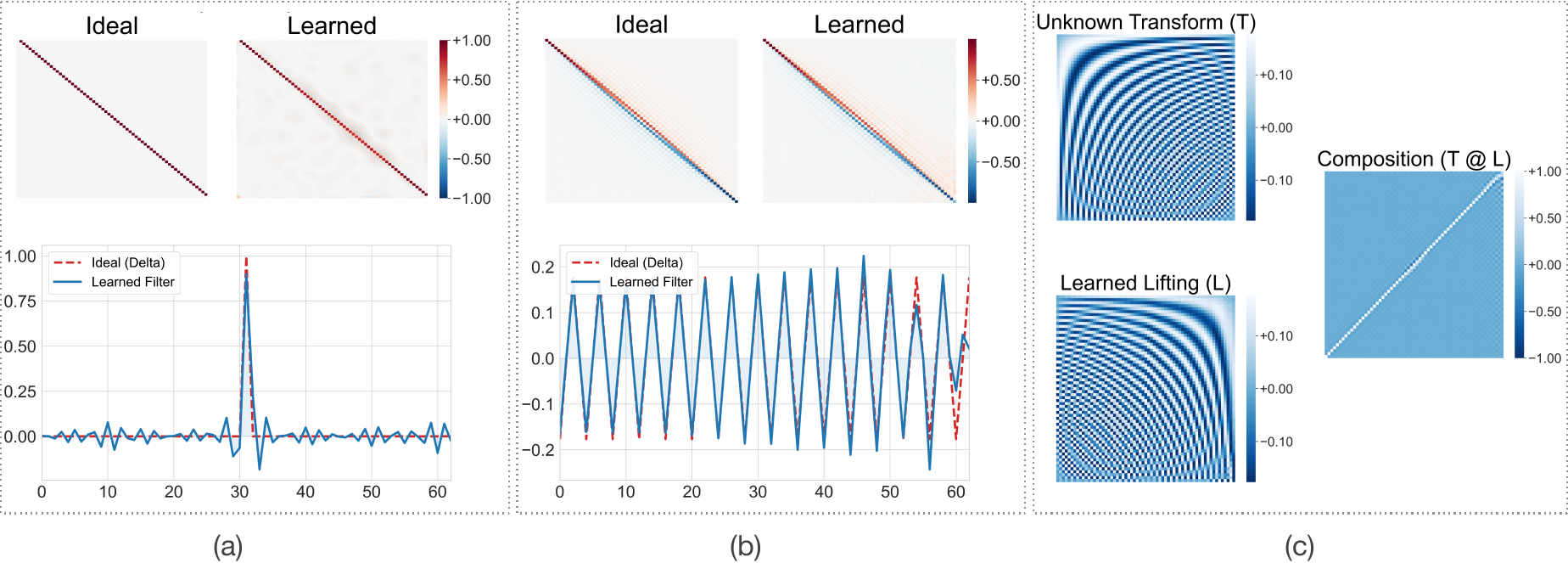}  
  \caption{\textbf{Waveform recovery experiments.}
\textbf{(a)} Native domain: learned actions and resolving filters match one-dimensional translations and point responses.
\textbf{(b)} DST-I domain: learned actions and filters transform accordingly, with point responses represented by sinusoidal filters.
\textbf{(c)} The learned lifting approximately inverts the DST-I transform, yielding an identity-like composition.}
  \label{fig:waveform_experiments}
\end{figure}

\begin{table}[t]
\centering
\caption{\textbf{Waveform and MNIST benchmarks.}
We report generator similarity $S_{0.75}$ and signal recovery $r$.}
\label{tab:waveform_mnist_benchmarks}

\begingroup
\setlength{\tabcolsep}{4pt}
\renewcommand{\arraystretch}{0.88}
\setlength{\aboverulesep}{0.25ex}
\setlength{\belowrulesep}{0.25ex}

\begin{tabular}{lccccccc}
\toprule
\multicolumn{8}{c}{\textbf{Waveform symmetry-generator discovery}} \\
\midrule
Dataset & $d$ & \multicolumn{2}{c}{Translation regime} &
\multicolumn{2}{c}{LieGAN $S_{0.75}$} &
\multicolumn{2}{c}{Ours $S_{0.75}$} \\
\midrule
GSN-G & $15$ & \multicolumn{2}{c}{FFT-based fractional} & \multicolumn{2}{c}{\textbf{0.9999}} & \multicolumn{2}{c}{0.9953} \\
GSN-G & $15$ & \multicolumn{2}{c}{Discrete latent}       & \multicolumn{2}{c}{0.4277}          & \multicolumn{2}{c}{\textbf{0.9933}} \\
GSN-G & $15$ & \multicolumn{2}{c}{Continuous latent}     & \multicolumn{2}{c}{0.4533}          & \multicolumn{2}{c}{\textbf{0.9862}} \\

\midrule
\multicolumn{8}{c}{\textbf{Waveform latent domain recovery}} \\
\midrule
Dataset & $d$ & Transform. & Gen. & \multicolumn{4}{c}{Signal recovery $r$} \\
\cmidrule(lr){5-8}
 & & & $S_{0.75}$ & Ours & GLASSO & CM & Random \\
\midrule
GSN-G & $63$ & Identity       & 0.961 & \textbf{0.850} & 0.812 & 0.495 & 0.284 \\
GSN-L & $63$ & Identity       & 0.964 & \textbf{0.981} & 0.487 & 0.272 & 0.176 \\
GSN-G & $63$ & DST-I          & 0.970 & \textbf{0.982} & 0.352 & 0.426 & 0.277 \\
GSN-L & $63$ & DST-I          & 0.959 & \textbf{0.997} & 0.184 & 0.263 & 0.172 \\
Ising & $33$ & Non-orthogonal & 0.956 & \textbf{0.960} & 0.352 & 0.595 & 0.332 \\

\midrule
\multicolumn{8}{c}{\textbf{Scrambled MNIST domain recovery}} \\
\midrule
Crop $(d)$ & Gen. & \multicolumn{6}{c}{Signal recovery $r$} \\
\cmidrule(lr){3-8}
 & $S_{0.75}$ & Ours & GLASSO & CM & UMAP & ISOMAP & Random \\
\midrule
$15{\times}15$ $(225)$ & 0.921/0.949 & \textbf{0.995} & 0.371 & 0.316 & 0.790 & 0.763 & 0.126 \\
$27{\times}27$ $(729)$ & 0.718/0.716 & \textbf{0.947} & 0.233 & 0.197 & 0.702 & 0.611 & 0.070 \\
\bottomrule
\end{tabular}
\endgroup
\end{table}

\subsection{Experiments over images}
\label{subsubsection:blind-permutation-recovery}

\textbf{Pixel shuffling.}
We evaluate on the MNIST ``bag-of-pixels'' problem~\cite{roux2007learning}. After zero-padding, we randomly crop $15\times15$ and $27\times27$ patches to ensure translation invariance. We then flatten and randomly permute the pixels. Our model achieves strong recovery ($r=0.995$, $0.947$; Table~\ref{tab:waveform_mnist_benchmarks}) while learning sparse orthogonal symmetry representations of size $225\times225$ and $729\times729$. Although $S_{0.75}$ decreases for $27\times27$, robust low-frequency scores ($S_{0.5}>0.96$) suggest that errors concentrate in high-frequency modes, which carry limited energy.

For benchmarking, we apply the graph structure learning methods and techniques proposed in \cite{roux2007learning}. The approach for graph structure learning algorithms is similar to the 1D domain discovery experiments; however, we align the precision or covariance matrix to the adjacency matrix of a 2D grid instead of a 1D grid. 

We use the methods proposed in \cite{roux2007learning} without modification. These methods treat the dataset as a structureless bag of pixels. They rely on manifold learning algorithms, such as ISOMAP and UMAP, to find embedding vectors for each pixel index before applying linear interpolation to reveal the image.

Our method significantly improves upon the strongest baseline, increasing accuracy from $0.790$ to $0.995$ for $15\times15$ crops, and from $0.702$ to $0.947$ for $27\times27$ crops. Although manifold-learning baselines achieve a better score compared to graph structure learning-based methods, it should also be considered that manifold learning methods only search for $O(d)$ pixel coordinate embeddings rather than general $O(d^2)$ linear transformations, giving them a significant advantage.

\textbf{Bit scrambling.}
To probe recovery limits, we simulate bit-scrambling encryption~\cite{ye2010image} by decomposing pixels into randomly permuted bits. This yields dimensions of $1800$ and $5832$ for the $15{\times}15$ and $27{\times}27$ crops, respectively. The bit-tensor representation casts nonlinear recovery as a high-dimensional linear inverse problem.

A learnable linear embedding reduces the bitstream to the original image dimension. During training, digit structure emerges from low-frequency fluctuations (Figure~\ref{figure:2d-permutation-group-convolution-tensor-snapshots}). The induced operator aggregates bits into pixels, restores spatial topology, and recovers bit-significance structure (Appendix, Figure~\ref{fig:from_bits_to_images_lifting_matrix_structure}). Recovery remains stable across scales ($r=0.944$, $0.941$), suggesting the first passive, unsupervised recovery of images from scrambled bitstreams.

\subsection{Learning rotational symmetry from mouse visual cortex signals}

\textbf{Visual coding dataset.}
To explore black-box biological signals, we use the Allen Brain Observatory Visual Coding dataset~\cite{siegle2021survey}. We focus on sinusoidal grating stimuli with six orientations: $\theta \in \{0^\circ, 30^\circ, \dots, 150^\circ\}$. Each sample records visual-cortex activity from $110$ neurons in response to one grating displayed on a screen, yielding approximately $1.1$k spike-rate vectors (detailed in Appendix~\ref{appsubsec:neural-experiment-details}).

\textbf{Homogeneous phase setting.}
This experiment matches the homogeneous setting of Section~\ref{problem-setting:homogeneous-case}. In this setting, the latent signal can be modeled as a delta spike on a circular domain. Although a delta spike is not well-defined on $L^2(G)$—rendering direct signal recovery intractable—its projection onto a low-frequency subspace constitutes a valid signal. Specifically, the lowest frequency subspace containing the fundamental harmonics is sufficient to reveal the latent signal. This is because a signal as simple as a delta spike can be characterized solely by its position (phase) on the circular domain.

Therefore, to extract the latent signal, we only need to recover its projection onto this fundamental harmonic subspace. Revealing the signal in this subspace corresponds to limiting the entropy-estimator rank to $3$. 

\begin{figure}[t!]
\centering

\begin{subfigure}[t]{0.275\textwidth}
    \centering
    \includegraphics[width=\linewidth]{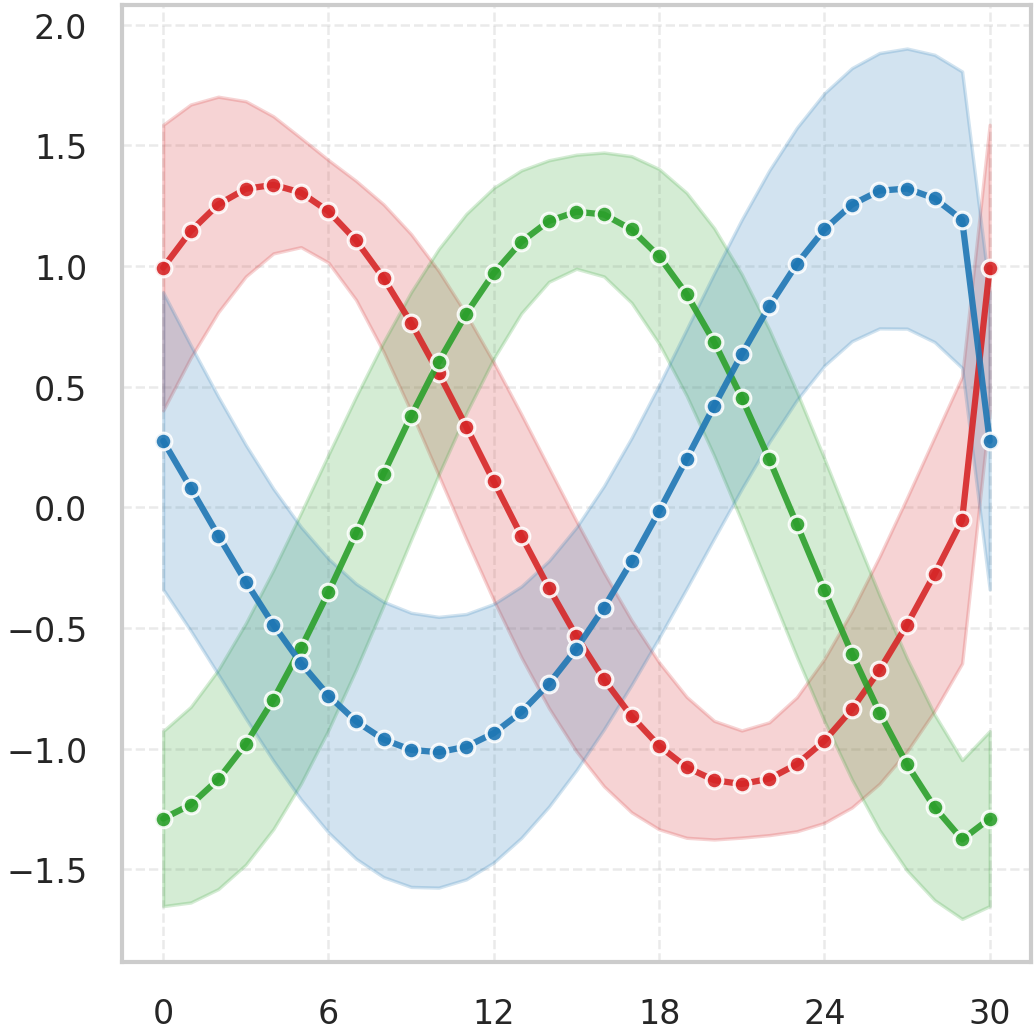}
    \caption{Model output.}
    \label{fig:response}
\end{subfigure}
\hspace{0.01\textwidth}
\begin{subfigure}[t]{0.275\textwidth}
    \centering
    \includegraphics[trim={0cm 0cm 0cm 0cm}, clip, width=\linewidth]{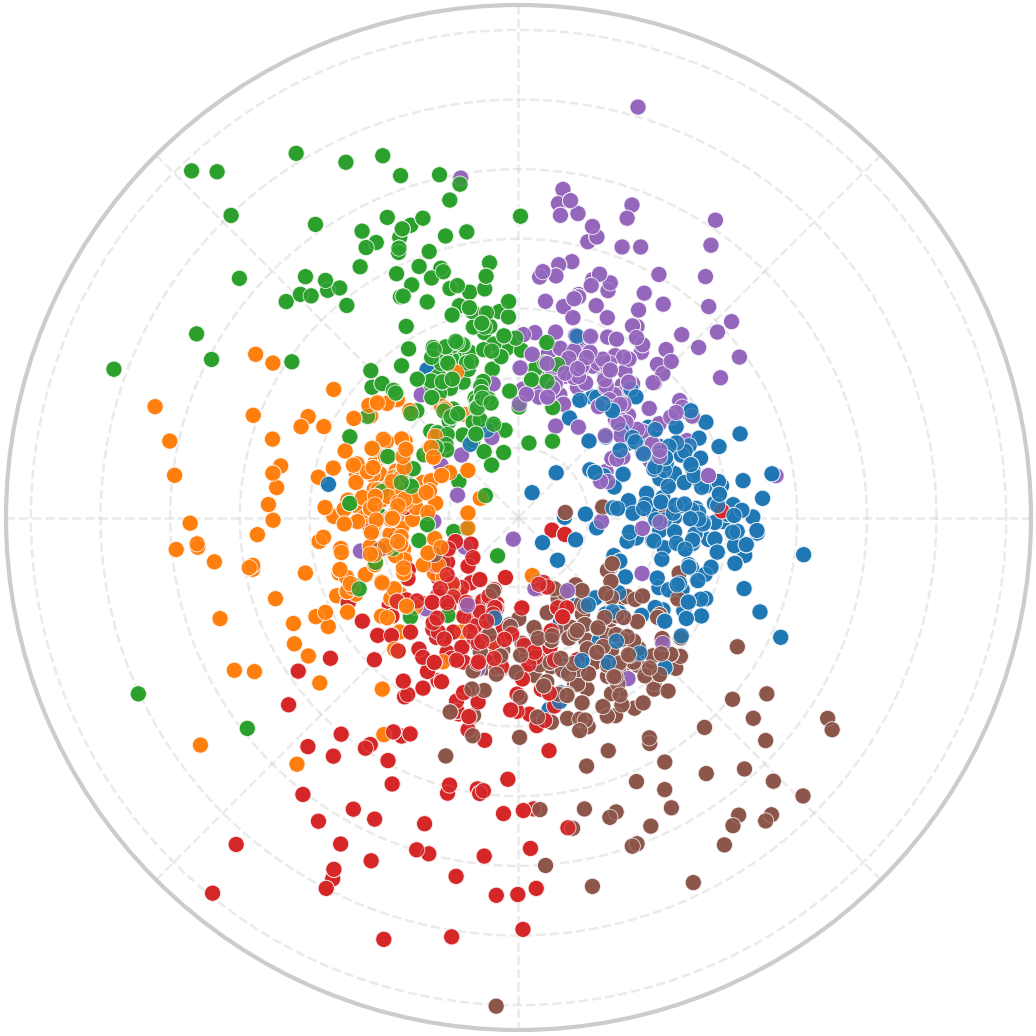}
    \caption{Polar code.}
    \label{fig:polar}
\end{subfigure}
\hspace{0.01\textwidth}
\begin{subfigure}[t]{0.275\textwidth}
    \centering
    \includegraphics[trim={1cm 1cm 0cm 0cm}, clip, width=\linewidth]{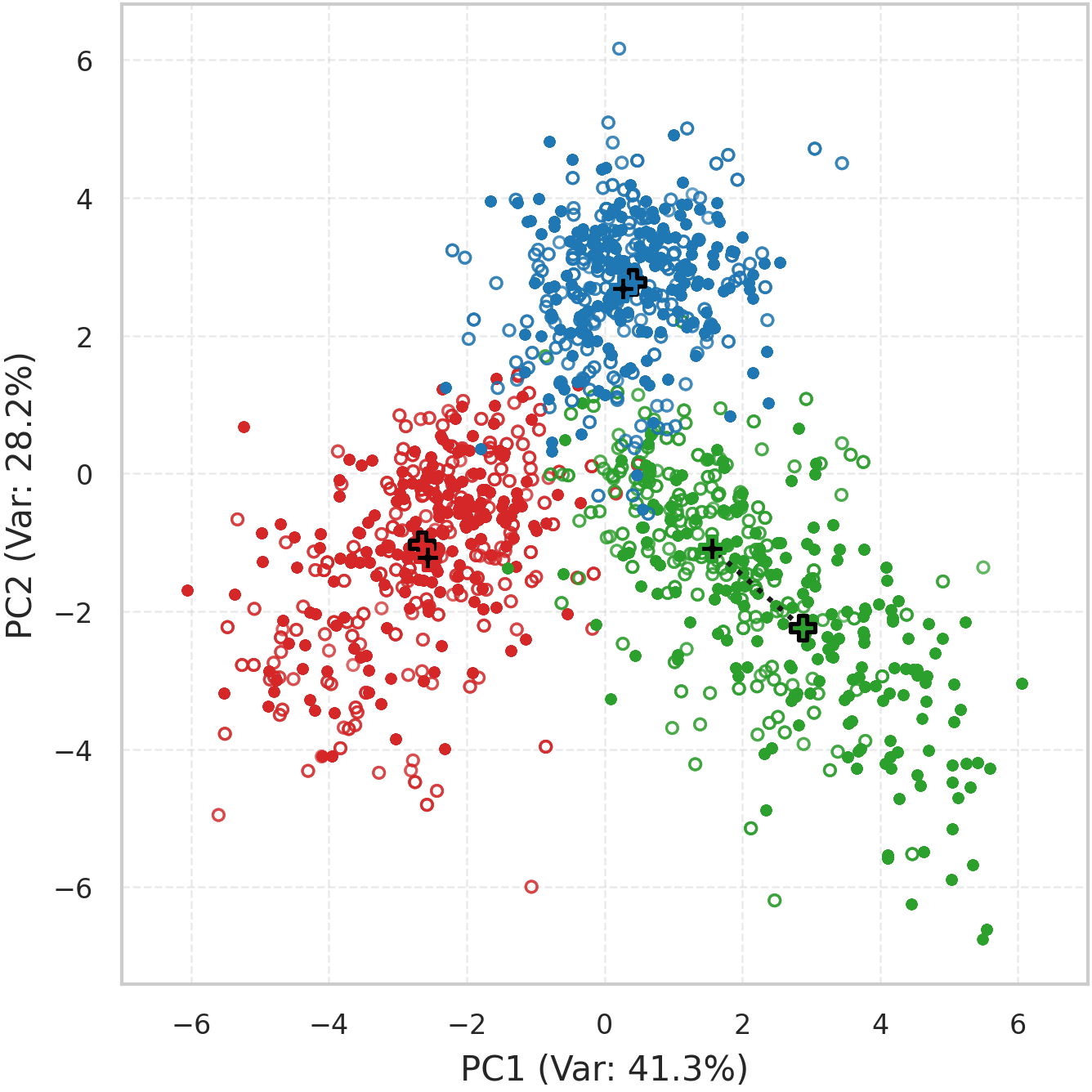}
    \caption{Real vs synthetic samples.}
    \label{fig:transport}
\end{subfigure}

\caption{\textbf{Learning rotation equivariance from neural signals.}
Samples are color-coded by stimulus orientation.
\textbf{(a--b)} The learned output forms a sinusoidal code whose phase recovers the circular orientation order, while amplitude captures response strength.
\textbf{(c)} In the first two principal components, real and synthetically shifted samples show substantial overlap.}
\label{fig:combined_analysis}
\end{figure}

\textbf{Experiment details.}
We convert neural action potentials into spike-rate vectors and set $\mathcal X_{\mathrm{aug}}=\mathcal X$, ensuring the learned generators act directly on the spike-rate vector space. After training, the model output resembles a sinusoid, whose phase and amplitude depend on the model input. The phase of this sinusoid corresponds to positions on the latent orientation circle, modulo 180° due to the mirror symmetry of the stimuli (Figure~\ref{fig:response}). In Figure~\ref{fig:polar}, we present a polar plot where the angular position of each sample is determined by the phase of the model output, and the distance from the origin corresponds to the amplitude. Samples belonging to the same stimulus consistently cluster within specific angular sectors, directly revealing the orientation observed by the mouse. 

For benchmarking, we map spike-rate vectors into 2D embeddings using either UMAP or ISOMAP to extract orientation information. First, we compute the average embedding position, $x_{\mathrm{avg}} = \mathbb{E}[x_i]$ and $y_{\mathrm{avg}} = \mathbb{E}[y_i]$, where $x_i$ and $y_i$ correspond to the embedding coordinates for the $i$-th sample. The orientation for each sample is then obtained via $\phi_i = \arctan2(y_i - y_{\mathrm{avg}}, x_i - x_{\mathrm{avg}})$.

We quantify the alignment between the true and blindly-estimated stimulus orientations using the circular correlation $r_{\mathrm{cc}}$~\cite{fisher1983correlation, jammalamadaka2001topics}. A fundamental challenge in neurobiology is that multiple factors of variation exist (e.g., the subject's excitement level or physical movement). Consequently, spike-rate vectors exhibit a high degree of sample-to-sample variability even for identical stimuli. Despite these challenges, our model achieves $r_{\mathrm{cc}}=0.851$, compared with $0.551$ for Isomap and $0.347$ for UMAP. This indicates that the symmetry prior successfully emphasizes orientation-dependent neural signals.

We also visually investigate the model's capacity for symmetry learning. Notably, symmetry discovery in this setting should be considered approximate, since it is not clear a priori whether physical stimulus rotations can be represented as linear operators acting on the spike-rate vector space. Under the linear action assumption, we evaluate the learned generator by exponentiating it to obtain a 30° rotation operator, which we then apply to spike-rate vectors from different orientations. In Figure~\ref{fig:transport}, solid dots correspond to real responses, while hollow dots are generated by applying the learned operator to samples with a relative -30° offset. Their overlap with target orientations indicates that the learned generator successfully approximates the expected symmetry action in the spike-rate vector space.

\subsection{Ablation studies}
We ablate stationarity, resolution, and InfoMax on GSN while keeping all other hyperparameters fixed (Table~\ref{tab:ablation}).

\textbf{Resolution is essential.}
Removing the resolution term (which minimizes total correlation) causes collapse on both Legendre and Gaussian signals ($S_{0.75}\approx 0$). This confirms that independence-promoting regularization is necessary to avoid trivial stationary solutions.

\textbf{Stationarity enforces global structure.}
Without stationarity, the model still recovers the Legendre generator ($S_{0.75}=0.997$). This is likely because localized Legendre bases provide strong locality cues. However, Gaussian recovery fails and signal correlation collapses in both cases, showing that stationarity is needed for globally consistent recovery.

\textbf{InfoMax improves stability.}
Removing InfoMax collapses the Legendre setting but improves Gaussian recovery. This indicates that the term mitigates degeneracy but can hurt when overweighted; thus, InfoMax is primarily useful as a moderate anti-collapse regularizer.

\begin{table}[htbp]
    \centering
    \caption{\textbf{Ablation on GSN.}
    Results on GSN variants, reporting band-limited similarity $S_{0.75}$ and recovery correlation $r$.}
    \label{tab:ablation}
    \setlength{\tabcolsep}{4pt}
    \renewcommand{\arraystretch}{0.92}
    \begin{tabular}{lcc}
        \toprule
        Removed term & Legendre $S_{0.75}/r$ & Gaussian $S_{0.75}/r$ \\
        \midrule
        None          & 0.964 / 0.981 & 0.961 / 0.850 \\
        Stationarity  & 0.997 / 0.012 & 0.781 / 0.031 \\
        Resolution    & 0.058 / 0.010 & $-0.048$ / 0.013 \\
        InfoMax       & 0.830 / 0.009 & \textbf{0.992 / 0.932} \\
        \bottomrule
    \end{tabular}
\end{table}
\section{Discussion}\label{sec:discussion}
In this study, we proposed a symmetry-based approach for recovering latent domain structures and signals. We supported the method theoretically through identifiability arguments and empirically through controlled and real-world experiments. Our benchmark experiments show that the method can recover latent domains and signals in a challenging regime that remains largely open: unknown linear corruptions that do not admit an explicit convolutional structure.

Moreover, moving beyond settings where symmetry actions are defined on structured observations, we show that the proposed method can discover latent symmetry actions from high-dimensional, unstructured inputs. This suggests that symmetry discovery can be used not only as a tool for exploiting known structure, but also as an exploratory tool for uncovering hidden organization in black-box systems. In this sense, the method can be applicable in domains such as biological neural recordings, where the underlying spatial or functional organization is not directly observed.

However, the proposed methodology has two limitations. First, in order to solve the otherwise intractable inverse problem, we inspire from the symmetry and locality properties of physical space, and assume that they manifest themselves statistically in the observed data. When this assumption is only partially satisfied, the model may map the latent domain to a representation that enforces these priors, thereby distorting aspects of the restored signal.

The second limitation is computational complexity. In particular, the required diagonalization procedures scale cubically and appear in two stages: symmetry-representation and output cross-covariance diagonalization. The cost of the former can be reduced by using a compressive learnable embedding, as demonstrated in our recovery of latent digits from bit-scrambled inputs. For the latter, randomized eigensolvers and scalable matrix factorization methods can be used to reduce computational costs.

Finally, we assumed that latent domains are either rectangular patches or tori. Extending the method to more general domain geometries is an important direction for future work. This will likely require discovering richer, possibly non-commutative symmetry structures, which remains a challenging but promising avenue.

\newpage

\bibliographystyle{unsrtnat}
\bibliography{biblio}

\newpage
\appendix

\section{Additional theoretical analysis}\label{additional-analysis}

\subsection{Resolution term and Markov structure}\label{analysis-of-resolution}
In this section, we analyze how the resolution term operates, restricting our setup to one-parameter groups for simplicity. In summary, we show that total correlation regularization, combined with architectural constraints, leads to the suppression of higher-order conditional mutual information terms, thereby favoring a low-order Markov structure.

\subsubsection{Total correlation decomposition}
In this step, we demonstrate that total correlation can be written as a sum of conditional mutual informations. Since an $n$-th order Markov structure satisfies $I(\rvector{X}|X_{[i:i+m]}) = 0$ for any $m > n$, this provides the foundational connection between the Markov structure and total correlation.

\begin{proposition}\label{prop:tc-cmi-decomposition}
Let \(\rvector{Y}=(Y_1,\dots,Y_d)\) be a \(d\)-dimensional random variable with finite entropies. Define
\[
    Y_{a:b}:=(Y_a,Y_{a+1},\dots,Y_b),
\]
with the convention that \(Y_{a:b}\) is empty whenever \(a>b\). Then the total correlation
\[
    \operatorname{TC}(\rvector{Y}) = \sum_{i=1}^d h(Y_i)-h(\rvector{Y})
\]
can be written as sum of pair-wise and conditional mutual informations
\[
    \operatorname{TC}(\rvector{Y}) = \sum_{i=2}^d I\!\left(Y_i;Y_{i-1}\right) 
    +
    \sum_{m=1}^{d-2} \sum_{i=m+2}^{d} I\!\left(Y_i; Y_{i-m-1}\mid Y_{i-m:i-1} \right).
\]
\end{proposition}

\begin{proof}
For \(m=0,\dots,d-1\), define the \(m\)-th order chain-rule approximation
\[
    h^{(m)}(\rvector{Y}) := \sum_{i=1}^{d} h\!\left(Y_i \mid Y_{\max(1,i-m):i-1}\right).
\]
With this convention,
\[
    h^{(0)}(\rvector{Y}) = \sum_{i=1}^d h(Y_i),
\]
while
\[
    h^{(d-1)}(\rvector{Y}) = \sum_{i=1}^d h(Y_i\mid Y_{1:i-1}) = h(\rvector{Y})
\]
by the chain rule. Therefore,
\[
    \operatorname{TC}(\rvector{Y}) = h^{(0)}(\rvector{Y})-h^{(d-1)}(\rvector{Y})
    =
    \sum_{m=0}^{d-2} \left[ h^{(m)}(\rvector{Y}) - h^{(m+1)}(\rvector{Y}) \right].
\]
For each \(m\), the difference is
\[
\begin{aligned}
    h^{(m)}(\rvector{Y})-h^{(m+1)}(\rvector{Y}) &=
    \sum_{i=m+2}^{d} \Big[ h(Y_i\mid Y_{i-m:i-1}) - h(Y_i\mid Y_{i-m-1:i-1}) \Big] \\
    &= \sum_{i=m+2}^{d} I\!\left(Y_i; Y_{i-m-1} \mid Y_{i-m:i-1} \right).
\end{aligned}
\]
Substituting this into the sum gives
\[
    \operatorname{TC}(\rvector{Y}) = 
    \sum_{m=0}^{d-2} \sum_{i=m+2}^{d} I\!\left(Y_i; Y_{i-m-1} \mid Y_{i-m:i-1} \right).
\]
If we isolate the term $m=0$ we have;
\[
    \operatorname{TC}(\rvector{Y})
    = 
    \sum_{i=2}^{d} I\!\left(Y_i; Y_{i-1} \right)
    +
    \sum_{m=1}^{d-2} \sum_{i=m+2}^{d} I\!\left(Y_i; Y_{i-m-1} \mid Y_{i-m:i-1} \right).
\]
\end{proof}
  
\subsubsection{Architectural lower bound on local dependence}
Without the constraints induced by the architecture, total correlation tends to eliminate all dependencies between the components of the model output. In this section, we show that complete elimination is impossible due to the architectural constraints imposed by the lifting operation. Specifically, the lifting operation places a lower bound on mutual information depending on their distance in the model output representation, which is induced by the discovered symmetry group.  

Before proceeding with the proof, we define the \emph{entropy power} $N_h(Y) := \frac{1}{2\pi e}\exp(2h(Y))$ for brevity. Specifically, for a scalar random variable \(Y\), \(N_h(Y)\) is the variance of a Gaussian random variable having the differential entropy $h(Y)$.

\begin{lemma}[Lifting-induced mutual information lower bound]
\label{lem:lifting-mi-bound}
Let \(\hat\rho:\mathbb{R}\to SO(d)\) be an orthogonal Lie group representation,
and let
\[
    Y_n = \vs{\Phi}^{\top}\hat\rho(-t_n)\rvector{X}, \qquad \|\vs{\Phi}\|_2=1, \qquad \mathbb{E}\|\rvector{X}\|_2^2=1 .
\]
Then, for any \(i,j\) such that \(h(Y_j)\) is finite,
\[
    I(Y_i;Y_j) \geq \frac12 \log \frac{N_h(Y_j)}{\left\| \hat\rho(t_i-t_j)-\identity\right\|_F^2}.
\]
\end{lemma}

\begin{proof}
As before,
\[
    \mathbb{E}|Y_j-Y_i|^2 \leq \left\| \hat\rho(t_i-t_j)-\identity \right\|_F^2 .
\]
Let
\[
    \varepsilon_{ij} := \left\| \hat\rho(t_i-t_j)-\identity \right\|_F^2 .
\]
Moreover,
\[
    \mathbb{E}\operatorname{Var}(Y_j\mid Y_i) \leq \mathbb{E}|Y_j-Y_i|^2 \leq \varepsilon_{ij}.
\]
By the maximum-entropy bound for a scalar random variable with fixed variance,
\[
    h(Y_j\mid Y_i) \leq \frac12\log(2\pi e\,\varepsilon_{ij}).
\]
Therefore,
\[
\begin{aligned}
    I(Y_i;Y_j) &= h(Y_j)-h(Y_j\mid Y_i) \\
    &\geq h(Y_j) - \frac12\log(2\pi e\,\varepsilon_{ij})  \\
    &= \frac12 \log \frac{\frac{1}{2\pi e}\exp(2h(Y_j))}{\varepsilon_{ij}} \\
    &= \frac12 \log \frac{N_h(Y_j)}{\left\| \hat\rho(t_i-t_j)-\identity \right\|_F^2} .
\end{aligned}
\]
\end{proof}

\subsubsection{Suppression of higher-order conditional dependencies}
In this section, we show that total correlation minimization suppresses high-order conditional dependencies, thereby favoring a finite-order Markov structure.

\begin{proposition}[Total correlation combined with lifting suppresses higher-order conditional dependencies]
\label{prop:resolution-suppresses-higher-order}
Let \(\rvector{Y}=(Y_1,\dots,Y_d)\) be the lifted output evaluated at equally spaced coordinates \(t_i=i\Delta\), where $\Delta \in \mathbb{R}^{+}$. Define the higher-order residual
\[
    R_{>1}(\rvector{Y}) := \sum_{m=1}^{d-2} \sum_{i=m+2}^{d} I\!\left(Y_i; Y_{i-m-1} \mid Y_{i-m:i-1} \right).
\]
Assume that the hypotheses of Lemma~\ref{lem:lifting-mi-bound} hold and that the entropy powers have a minima $\underline N$, such 
that $N_h(Y_i) \geq \underline N, \qquad i=2,\dots,d$. 
Then
\[
    R_{>1}(\rvector{Y}) \leq \operatorname{TC}(\rvector{Y}) - \frac{d-1}{2} \log \frac{\underline N}{\left\|\hat\rho(\Delta)-\identity\right\|_F^2}.
\]
In particular, if the lifted process is shift-invariant so that \(N_h(Y_i)=N_0\) for all \(i\), then we may take \(\underline N=N_0\).
\end{proposition}

\begin{proof}
By Proposition~\ref{prop:tc-cmi-decomposition}, the total correlation decomposes as
\[
    \operatorname{TC}(\rvector{Y}) = \sum_{i=2}^{d} I(Y_i;Y_{i-1}) + R_{>1}(\rvector{Y}) .
\]
Therefore,
\[
    R_{>1}(\rvector{Y}) = \operatorname{TC}(\rvector{Y}) - \sum_{i=2}^{d} I(Y_i;Y_{i-1}).
\]
Since \(t_i=i\Delta\), Lemma~\ref{lem:lifting-mi-bound} gives, for every \(i=2,\dots,d\),
\[
    I(Y_i;Y_{i-1}) \geq \frac12 \log \frac{N_h(Y_i)}{\left\|\hat\rho(\Delta)-\identity\right\|_F^2}.
\]
Substituting this lower bound into the identity for \(R_{>1}\) gives
\[
    R_{>1}(\rvector{Y}) \leq \operatorname{TC}(\rvector{Y}) - \frac12 \sum_{i=2}^{d} \log \frac{ N_h(Y_i)}{\left\|\hat\rho(\Delta)-\identity \right\|_F^2}.
\]
Using \(N_h(Y_i)\geq \underline N\) for \(i=2,\dots,d\), we obtain
\[
    R_{>1}(\rvector{Y}) \leq \operatorname{TC}(\rvector{Y}) - \frac{d-1}{2} \log \frac{\underline N}{\left\|\hat\rho(\Delta)-\identity\right\|_F^2}.
\]
If the lifted process is shift-invariant, then all one-dimensional marginals are identical. Hence \(N_h(Y_i)=N_0\), and we may take \(\underline N=N_0\).
\end{proof}

\begin{corollary}[First-order Markov limit under optimal TC minimization]
\label{cor:first-order-markov-limit}
Under the assumptions of Proposition~\ref{prop:resolution-suppresses-higher-order}, suppose additionally that the lifted process is shift-invariant, so that
\[
    N_h(Y_i)=N_0, \qquad i=1,\dots,d .
\]
Define
\[
    B_\Delta := \frac{d-1}{2} \log \frac{N_0}{\left\|\hat\rho(\Delta)-\identity\right\|_F^2} .
\]
If TC minimization drives the lifted representation to the architectural lower floor,
\[
    \operatorname{TC}(\rvector{Y}) - B_\Delta \longrightarrow 0,
\]
then
\[
    R_{>1}(\rvector{Y}) \longrightarrow 0 .
\]
Consequently, every higher-order conditional mutual information term in \(R_{>1}\) vanishes in the limit, and the lifted coordinates satisfy the first-order Markov conditional-independence relations along the recovered ordering.
\end{corollary}

\begin{proof}
By Proposition~\ref{prop:resolution-suppresses-higher-order}, shift-invariance
gives
\[
    0 \leq R_{>1}(\rvector{Y}) \leq \operatorname{TC}(\rvector{Y}) - B_\Delta .
\]
The left inequality holds because \(R_{>1}\) is a sum of conditional mutual information terms. Hence, if
\[
    \operatorname{TC}(\rvector{Y}) - B_\Delta \to 0,
\]
the squeeze theorem gives
\[
    R_{>1}(\rvector{Y}) \to 0 .
\]
Since \(R_{>1}\) is a finite sum of nonnegative conditional mutual informations, each individual higher-order term also vanishes.
\end{proof}

\subsection{Resolution term and compact recovery kernels}
In this section, we show that if both the latent field and the model output constitute a low-order Markov random field, the recovery filter is inherently compact. To demonstrate this, we leverage the Hessian sparsity property of Markov random fields. Specifically, a $p$-th order MRF has a $p$-banded Hessian matrix. Therefore, by assuming that both the input and output are Markov fields of limited (though not necessarily identical) orders, and by investigating the Hessians of both fields, we conclude that the filter is compact.

\subsubsection{Hessian transformation under linear maps}
We first record how the energy Hessian of a Markov random field transforms under an invertible linear map. Let \(\rvector{F}\in\mathbb{R}^N\) be a real-valued discretized Markov random field, and let
\[
    \rvector{Y}=K\rvector{F},
\]
where \(K:\mathbb{R}^N\to\mathbb{R}^N\) is invertible.

\begin{lemma}[Hessian relation for a transformed Markov random field]
\label{lem:hessian-relation-for-markovs}
Let \(\rvector{F}\in\mathbb{R}^N\) have density
\[
    p_F(\mathbf f) \propto \exp[-E_F(\mathbf f)] ,
\]
where \(E_F\) is twice differentiable. Let
\[
    \rvector{Y}=K\rvector{F},
\]
with \(K\) invertible, and assume that \(\rvector{Y}\) has twice differentiable
energy \(E_Y\). Then
\[
    \nabla^2 E_F(\mathbf f) = K^\top \nabla^2 E_Y(K\mathbf f) K .
\]
\end{lemma}

\begin{proof}
Since \(\mathbf y=K\mathbf f\), the change-of-variables formula gives
\[
    E_F(\mathbf f) = E_Y(K\mathbf f) + \textnormal{constant}.
\]
Differentiating twice with respect to \(\mathbf f\) gives
\[
    \nabla^2 E_F(\mathbf f) = K^\top \nabla^2 E_Y(K\mathbf f) K .
\]
\end{proof}

\subsubsection{Banded Hessians imply compact filters}
Here, we show how the banded Hessians of Markov random fields force the convolutional filter to take a compact form. Therefore, this proposition serves as the final step in demonstrating how the resolution term induces compactness for the recovery kernel.

\begin{proposition}
    Assume that real-valued N-dimensional random vectors $\rvector{F}$, $\rvector{Y} = K\rvector{F}$ satisfy p-order Markov field properties while $K: \mathbb{R}^N \to \mathbb{R}^N$ is invertible. And assume that diagonal entries of \(\nabla^2 E_Y(K\mathbf f)\) 
    can be varied independently. Then each row of \(K\) has compact support of diameter at most \(p\).
\end{proposition}

\begin{proof}
    For brevity, lets define $A(f) := \nabla^2 E_F(\mathbf f)$ and $B(f) := \nabla^2 E_Y(K\mathbf f)$. Then Hessian transformation in lemma~\ref{lem:hessian-relation-for-markovs} can be written as; 
    \[
    A(f) = K^\top B(f) K .
    \]
    Now, based on the assumption, we can vary $l$'th diagonal component of $B(f)$ independently. Then;
    \[
    {[K^\top]}_{il} K_{lj} \delta B(f)_l = 0, \qquad \text{for } |i-j| > p. 
    \] 
    This implies that $K_{li} K_{lj} = 0$ for any $|i-j| > p$ for any $l$. Therefore, each row of \(K\) has compact support of diameter at most \(p\).
\end{proof}

\paragraph{Remark.}
Aside from the Markov properties of both the latent field and the model output, this proof assumes that the diagonal entries of \(\nabla^2 E_Y(K\mathbf f)\) can be varied independently (i.e., non-degenerate curvature variation). Interestingly, this corresponds to non-Gaussianity for 0-th order Markov fields, which is the principle behind Independent Component Analysis (ICA). For higher-order Markov fields, the non-Gaussianity condition shifts to possessing local curvature variations. Therefore, this can be viewed as a finite-range Markov analogue of the ICA principle.

\subsection{Lifting and oracle deconvolution}
\label{app:lifting-oracle-deconvolution}

\subsubsection{Oracle deconvolution on the torus}
\begin{proposition}[Deconvolution on the torus]
\label{prop:oracle-deconvolution}
Let \(G=\mathbb T^P\), and assume that \(f\in\mathcal F_B\) is bandlimited over the Fourier modes of the torus \(\mathcal B\subset\mathbb Z^P\). Suppose the augmented observation satisfies
\[
    \ve z
    =
    \int_{\mathbb T^P} f(\ve s)\rho(\ve s)\ve h_0\,d\mu(\ve s),
\]
where \(\rho\) is diagonalized by the basis \(\{\ve u_{\ve m}\}_{\ve m\in\mathcal B}\), with $\rho(\ve t)\ve u_{\ve m} = e^{-i\ve m\cdot\ve t}\ve u_{\ve m}$.
If \(h_{\ve m}:=\ve u_{\ve m}^\dagger\ve h_0\neq 0\) for all \(\ve m\in\mathcal B\), then there exists a filter \(\ve w_\star\) such that
\[
    f(\ve t)=\ve w_\star^\dagger\rho(-\ve t)\ve z .
\]
\end{proposition}

\begin{proof}\label{proof:oracle-deconvolution}
First, define the projected signal
\[
    q(\ve t)
    :=
    \ve h_0^\dagger \rho(-\ve t)\ve z .
\]
Substituting the forward model gives
\begin{align}
q(\ve t)
&=
\ve h_0^\dagger \rho(-\ve t)
\left[
    \int_{\mathbb T^P} f(\ve s)\rho(\ve s)\ve h_0\,d\mu(\ve s)
\right] \nonumber \\
&=
\int_{\mathbb T^P}
f(\ve s)
\underbrace{
\ve h_0^\dagger \rho(\ve s-\ve t)\ve h_0
}_{k(\ve s-\ve t)}
\,d\mu(\ve s),
\label{eq:projection-as-convolution}
\end{align}
where \(k(\ve r):=\ve h_0^\dagger\rho(\ve r)\ve h_0\). Thus, \(q\) is a convolutional smoothing of \(f\).

Now write
\[
    \ve h_0
    =
    \sum_{\ve m\in\mathcal B}
    h_{\ve m}\ve u_{\ve m},
    \qquad
    h_{\ve m}:=\ve u_{\ve m}^\dagger\ve h_0 .
\]
Using the diagonalization of \(\rho\), the forward model implies
\[
    \ve u_{\ve m}^\dagger \ve z
    =
    h_{\ve m}\tilde f_{\ve m}.
\]
Equivalently, the Fourier coefficients of the projected signal satisfy
\[
    \tilde q_{\ve m}
    =
    |h_{\ve m}|^2\tilde f_{\ve m}.
\]
Since \(h_{\ve m}\neq 0\) on \(\mathcal B\), the latent Fourier coefficients are recovered by
\[
    \tilde f_{\ve m}
    =
    \frac{\tilde q_{\ve m}}{|h_{\ve m}|^2}
    =
    \frac{h_{\ve m}^\ast\,\ve u_{\ve m}^\dagger \ve z}
    {|h_{\ve m}|^2}.
\]
Applying the inverse Fourier transform yields
\begin{align}
    f(\ve t)
    &=
    \sum_{\ve m\in\mathcal B}
    \frac{h_{\ve m}^\ast\,\ve u_{\ve m}^\dagger \ve z}
    {|h_{\ve m}|^2}
    e^{i\ve m\cdot\ve t} \nonumber \\
    &=
    \left(
    \sum_{\ve m\in\mathcal B}
    \frac{h_{\ve m}^\ast}{|h_{\ve m}|^2}
    \ve u_{\ve m}^\dagger
    \right)
    \rho(-\ve t)\ve z \nonumber \\
    &=
    \ve w_\star^\dagger\rho(-\ve t)\ve z .
\end{align}
This proves the claim.
\end{proof}

The learned lifting operator in Section~\ref{sec:methodology} replaces the oracle pair \((\rho,\ve w_\star)\) with the learned representation and resolving filter \((\hat\rho,\ve w_0)\).

\subsection{Identifiability and kernel-recovery proofs}
\label{app:identifiability-kernel-proofs}

\subsubsection{Compact support and flat spectrum imply a delta kernel}
\begin{proof}\label{proof:compact-support-flat-filter-implies-delta}
Align the minimal support so that \(\kappa_n=0\) for \(n\notin[1,q]\), and suppose that \(2\leq q\leq N/2\). Since the support is minimal, \(\kappa_1,\kappa_q\neq 0\). The circular autocorrelation at shift \(q-1\) is
\[
    r(q-1)
    =
    \sum_{n=1}^N \kappa_n \kappa_{(n+q-1) \bmod N}
    =
    \kappa_1\kappa_q .
\]
Since a flat spectrum requires \(r(m)=0\) for all \(m \in [1, N-1]\), this leads to a contradiction. The only viable support is \(q=1\), so \(\kappa\) is a shifted and scaled Kronecker delta.
\end{proof}

\subsubsection{Shift-invariant kernels identify the symmetry action}
\begin{proof}\label{proof:symmetry_discovery}
Shift-invariance of the kernel gives
\[
    \ve{w}_0^\top\hat\rho(-\ve{t})
    \bigl(\rho(\ve{r})-\hat\rho(\ve{r})\bigr)
    \rho(\ve{s})\ve{h}_0
    =
    0 .
\]
Since
\[
    \mathcal X_{\mathrm{obs}}
    =
    \operatorname{span}\{\rho(\ve{s})\ve{h}_0:\ve{s}\in \mathbb{R}^P\}
    =
    \operatorname{span}\{\hat\rho(\ve{t})\ve{w}_0:\ve{t}\in \mathbb{R}^P\},
\]
we have
\[
    \ve{\hat z}^{\top}
    \bigl(\rho(\ve{r})-\hat\rho(\ve{r})\bigr)
    \ve{z}
    =
    0
\]
for all \(\ve{\hat z}, \ve{z} \in \mathcal X_{\mathrm{obs}}\). Therefore, \(\hat\rho(\ve{r})=\rho(\ve{r})\) on \(\mathcal X_{\mathrm{obs}}\).
\end{proof}

\subsubsection{Translation-only symmetries imply equivariance}
\begin{proof}\label{proof:translation-only-symmetries-imply-equivariance}
Since \(Y=\mathcal K F\) and \(\mathcal K\) is injective on the relevant signal subspace, we may write \(F=\mathcal K^{-1}Y\). Stationarity gives
\[
    \mathcal K^{-1}Y \overset{d}{=} T_{\ve{\alpha}} \mathcal K^{-1}Y
    \implies
    Y \overset{d}{=} \mathcal K T_{\ve{\alpha}}\mathcal K^{-1} Y .
\]
Thus \(\mathcal K T_{\ve{\alpha}}\mathcal K^{-1}\) is a distributional symmetry of the lifted field. By assumption, every distributional symmetry of \(Y\) is a lifted translation, so there exists a vector-valued function \(\ve{\beta} : \mathbb{R}^P \to \mathbb{R}^P\) such that
\[
    \mathcal K T_{\ve{\alpha}}\mathcal K^{-1}
    =
    T^{(Y)}_{\ve{\beta}(\ve{\alpha})}.
\]
Multiplying on the right by \(\mathcal K\) yields
\[
    \mathcal K T_{\ve{\alpha}}
    =
    T^{(Y)}_{\ve{\beta}(\ve{\alpha})}\mathcal K .
\]
\end{proof}

\subsubsection{Equivariance implies a shift-invariant recovery kernel}
\begin{proof}\label{proof:recovery-kernel-shift-invariance}
Equivariance gives
\[
    \int_G k(\ve{s},\ve{t})f(\ve{s}-\ve{\alpha})\,d\mu(\ve{s})
    =
    \int_G k(\ve{s},\ve{t}-\ve{\beta}(\ve{\alpha}))f(\ve{s})\,d\mu(\ve{s}).
\]
After the change of variables \(\ve{s}\mapsto \ve{s}+\ve{\alpha}\), equality for
all \(f\) gives
\[
    k(\ve{s}+\ve{\alpha},\ve{t})
    =
    k(\ve{s},\ve{t}-\ve{\beta}(\ve{\alpha})).
\]
Taking \(\ve{\alpha}=\ve{\beta}^{-1}(\ve{t})\) yields
\[
    k(\ve{s}+\ve{\beta}^{-1}(\ve{t}),\ve{t})
    =
    k(\ve{s},\ve{0}).
\]
Equivalently,
\[
    k(\ve{s},\ve{t})
    =
    k(\ve{s}-\ve{\beta}^{-1}(\ve{t}),\ve{0}).
\]
Defining
\[
    \kappa(\ve{r}) := k(\ve{r},\ve{0})
\]
proves the claim.
\end{proof}

\subsubsection{InfoMax implies a flat kernel spectrum}
\begin{proof}\label{proof:infomax-implies-flat-kernel-spectrum}
First, we note that the stationarity of the latent field implies that its covariance matrix is diagonalized by the Fourier modes. Denoting the eigenvalues of the covariance matrix as \(\lambda_m = |\tilde{\kappa}_m|^2 s_m\), the expression for the Gaussian entropy becomes
\[
    h(Y)
    =
    \frac{1}{2} \sum_{m=0}^{N-1} \log |\tilde{\kappa}_m|^2
    +
    \frac{1}{2} \sum_{m=0}^{N-1} \log s_m
    +
    \mathrm{constant}.
\]
Defining \(q_m = |\tilde{\kappa}_m|^2\), maximizing the output entropy subject to \(\sum_{m=0}^{N-1}q_m=C\) yields \(q_m = C/N\). Therefore, the kernel has a flat spectrum \(|\tilde{\kappa}_m|^2 = C/N\).
\end{proof}

\subsection{Implicit energy constraint of the kernel}
In this section, we demonstrate the implicit \(L^2\) constraint on the recovery kernel
\[
    k(\ve{s},\ve{t})
    =
    \ve w_0^\top \hat\rho(-\ve t)\rho(\ve s)\ve h_0 .
\]
For brevity, we investigate a one-parameter toroidal domain.

First, we specify spectral properties of the point-response vector \(\ve h_0\) under the delta-autocorrelation assumption, which follows naturally from its local properties. Then, under the assumption that the learned action has recovered the true symmetry, we show that this leads to a recovery kernel with fixed energy.

In Section~\ref{sec:identifiability}, we show that this property also leads to a recovery kernel with a flat spectrum under the InfoMax loss.

\subsubsection{Flat spectrum of the point-response vector}
\label{app:flat-spectrum-point-response}
In this section, we investigate the spectral properties of the point-response vector, which is how a latent point source is represented in the augmented sample space.

\begin{lemma}[Delta autocorrelation implies a flat spectrum]
Let \(G_N=\{r_n\}_{n=0}^{N-1}\) be a finite discretized toroidal translation group. Assume that \(\rho:G_N\to \mathrm U(N)\) is a multiplicity-free unitary representation diagonalized by an orthonormal character basis
\(\{u_l\}_{l=0}^{N-1}\):
\[
    \rho(r_n)u_l=\chi_l(r_n)u_l ,
\]
where \(\{\chi_l\}_{l=0}^{N-1}\) are distinct characters of \(G_N\). If \(\ve h_0 \in\mathbb C^N\) satisfies
\[
    \ve h_0^\dagger\rho(r_n)\ve h_0=\delta_{n0}
    \qquad
    \text{for all } n=0,\dots,N-1,
\]
then
\[
    |u_l^\dagger\ve h_0|^2=\frac{1}{N}
    \qquad
    \text{for all } l=0,\dots,N-1 .
\]
\end{lemma}

\begin{proof}
Write \(\ve h_0=\sum_l h_l u_l\), where \(h_l=u_l^\dagger\ve h_0\). Since \(\rho\) is diagonal in the character basis,
\[
    a_n
    :=
    \ve h_0^\dagger\rho(r_n)\ve h_0
    =
    \sum_{l=0}^{N-1} |h_l|^2 \chi_l(r_n).
\]
Thus \(a_n\) is the inverse group Fourier transform of the spectral weights \(p_l:=|h_l|^2\). By Fourier inversion on the finite group,
\[
    p_m
    =
    \frac{1}{N}
    \sum_{n=0}^{N-1} a_n \chi_m(r_n)^* .
\]
Using the assumption \(a_n=\delta_{n0}\), we obtain
\[
    p_m
    =
    \frac{1}{N}
    \sum_{n=0}^{N-1} \delta_{n0}\chi_m(r_n)^*
    =
    \frac{1}{N},
\]
because every character satisfies \(\chi_m(r_0)=1\). Therefore
\[
    |u_m^\dagger\ve h_0|^2
    =
    |h_m|^2
    =
    \frac{1}{N}.
\]
Since \(m\) was arbitrary, \(\ve h_0\) is spectrally flat in the character basis.
\end{proof}

\subsubsection{Fixed energy of the recovery kernel}
After the learned action has aligned with the true translation action, the lifting kernel becomes shift-invariant:
\[
    k(\ve{s},\ve{t})=\kappa(\ve{s}-\ve{t}).
\]
In the discrete toroidal setting, write \(r_n=n\Delta\) and define
\[
    \kappa_n := \kappa(r_n) = \ve w_0^\dagger \rho(r_n)\ve h_0 .
\]
Thus \(\kappa_n\) is the response of the learned filter to translated copies of the point-response vector. The following lemma shows that the total energy of this recovered convolution is fixed by normalization; the learned filter can change where the energy is placed, but not how much energy the kernel has.

\begin{lemma}[Recovery kernel has a fixed energy]
\label{lem:recovery-kernel-has-fixed-energy}
Let \(G_N=\{r_n=n\Delta\}_{n=0}^{N-1}\) be a discretized toroidal translation group with periodic indexing. Assume that \(\rho:G_N\to \mathrm U(N)\) is diagonalized by the character basis \(\{u_l\}_{l=0}^{N-1}\). Let
\[
    \kappa_n = \ve w_0^\dagger \rho(r_n)\ve h_0
\]
be the convolutional kernel induced by the learned filter \(\ve w_0\) and the point-response vector \(\ve h_0\).

Assume that the point-response vector has delta autocorrelation,
\[
    \ve h_0^\dagger \rho(r_n)\ve h_0 = \delta_{n0},
    \qquad n=0,\dots,N-1 .
\]
Then, by Lemma~\ref{app:flat-spectrum-point-response}, its spectrum is flat:
\[
    |u_l^\dagger \ve h_0|^2 = \frac{1}{N},
    \qquad l=0,\dots,N-1 .
\]
If the learned filter is normalized,
\[
    \|\ve w_0\|_2^2=1,
\]
then the recovered convolution has fixed average energy:
\[
    \frac{1}{N} \sum_{n=0}^{N-1} |\kappa_n|^2 = \frac{1}{N}.
\]
Equivalently, its total discrete energy is
\[
    \sum_{n=0}^{N-1} |\kappa_n|^2 = 1 .
\]
\end{lemma}

\begin{proof}
Expand the learned filter and the point-response vector in the character basis:
\[
    \ve w_0 = \sum_{l=0}^{N-1}w_l u_l,
    \qquad
    \ve h_0 = \sum_{l=0}^{N-1}h_l u_l,
\]
where
\[
    w_l=u_l^\dagger\ve w_0,
    \qquad
    h_l=u_l^\dagger\ve h_0 .
\]
Since \(\rho(r_n)\) is diagonal in this basis,
\[
    \rho(r_n)u_l=\chi_l(r_n)u_l .
\]
Therefore,
\[
    \kappa_n
    =
    \ve w_0^\dagger\rho(r_n)\ve h_0
    =
    \sum_{l=0}^{N-1} w_l^\ast h_l \chi_l(r_n).
\]
So the Fourier coefficient of the convolutional kernel at frequency \(l\) is
\[
    \widehat{\kappa}_l = w_l^\ast h_l .
\]

Using orthogonality of the character basis, or equivalently Parseval's identity on the finite torus,
\[
    \frac{1}{N} \sum_{n=0}^{N-1} |\kappa_n|^2
    =
    \sum_{l=0}^{N-1} |\widehat{\kappa}_l|^2 .
\]
Substituting \(\widehat{\kappa}_l=w_l^\ast h_l\), we obtain
\[
    \frac{1}{N} \sum_{n=0}^{N-1} |\kappa_n|^2
    =
    \sum_{l=0}^{N-1} |w_l|^2|h_l|^2 .
\]
The delta-autocorrelation condition on \(\ve h_0\) implies, by Lemma~\ref{app:flat-spectrum-point-response}, that
\[
    |h_l|^2=\frac{1}{N}
    \qquad
    \text{for all } l .
\]
Hence
\[
    \frac{1}{N} \sum_{n=0}^{N-1} |\kappa_n|^2
    =
    \frac{1}{N} \sum_{l=0}^{N-1} |w_l|^2 .
\]
Finally, since the learned filter is normalized,
\[
    \sum_{l=0}^{N-1}|w_l|^2
    =
    \|\ve w_0\|_2^2
    =
    1 .
\]
Therefore,
\[
    \frac{1}{N} \sum_{n=0}^{N-1} |\kappa_n|^2
    =
    \frac{1}{N}.
\]
Multiplying both sides by \(N\) gives
\[
    \sum_{n=0}^{N-1} |\kappa_n|^2 = 1 .
\]
\end{proof}

\subsection{Properties of the lifting operation}\label{subsection:lifting-props}
We next discuss some properties of the lifting operation for a general resolving filter \(\ve w\), with \(\ve w=\ve w_0\) in the trained model. Define
\[
    (\mathcal T_{\ve w}\ve z)(\ve t)
    :=
    \ve w^\top\hat\rho(-\ve t)\ve z .
\]

\textbf{Mapping samples to continuous functions:}
Let \(\ve z\in\mathcal X_{\mathrm{aug}}\) be an augmented sample, and let \(\ve t_1,\ve t_2\in \mathbb{R}^P\) be two points in the group parameter space.
Then, the values \(y_1 := (\mathcal T_{\ve w}\ve z)(\ve t_1)\) and \(y_2 := (\mathcal T_{\ve w}\ve z)(\ve t_2)\) of the lifted function at \(\ve t_1\) and \(\ve t_2\) cannot be too different when \(\ve t_1\) and \(\ve t_2\) are close. Namely,
\[
    |y_1-y_2|
    =
    \left|
    \ve w^\top\hat\rho(-\ve t_1)\ve z
    -
    \ve w^\top\hat\rho(-\ve t_2)\ve z
    \right|
    \leq
    \|\ve w^\top\|\,
    \|\hat\rho(-\ve t_1)-\hat\rho(-\ve t_2)\|\,
    \|\ve z\| .
\]
where \(\|\hat\rho(-\ve t_1)-\hat\rho(-\ve t_2)\|\) denotes a suitable norm. In the practical case of finite-dimensional vectors, \(\hat\rho\) is a matrix representation, and one can take the Frobenius norm. The continuity of the exponential map ensures that \(\|\hat\rho(\ve t_1)-\hat\rho(\ve t_2)\|_F \to 0\) when \(\ve t_1 \to \ve t_2\), showing that the lifting operation maps \(\ve z\) to a continuous function over group elements.

\subsection{Transformation of the point-source response with invertible maps}\label{app:locality-transformation}
Based on the data model proposed in Section~\ref{sec:problem-setting}, applying an invertible transformation changes both the point-source response \(\ve h_0\) and the symmetry representation.

To see this, let \(\mathbf A\in\mathbb R^{d_{\mathrm{aug}}\times d_{\mathrm{aug}}}\) be invertible and define \(\tilde{\ve z}:=\mathbf A\ve z\). Starting from the forward model,
\begin{align}\label{eq:change-of-basis}
\tilde{\ve z}
=
\mathbf A\ve z
&=
\int_G f(v)\,\mathbf A\rho(v)\ve h_0\,d\mu(v) \\
&=
\int_G f(v)\,\mathbf A\rho(v)\mathbf A^{-1}\mathbf A\ve h_0\,d\mu(v) \\
&=
\int_G f(v)\,\tilde{\rho}(v)\tilde{\ve h}_0\,d\mu(v),
\end{align}
where
\[
    \tilde{\rho}(v) := \mathbf A \rho(v)\mathbf A^{-1},
    \qquad
    \tilde{\ve h}_0 := \mathbf A \ve h_0 .
\]
\section{Implementation details}\label{app:implementation-details}
\subsection{Auxiliary network}\label{subsection:auxiliary-network}

\begin{figure}
\begin{subfigure}[t]{0.4\linewidth}
    \centering
    \includegraphics[width=\linewidth]{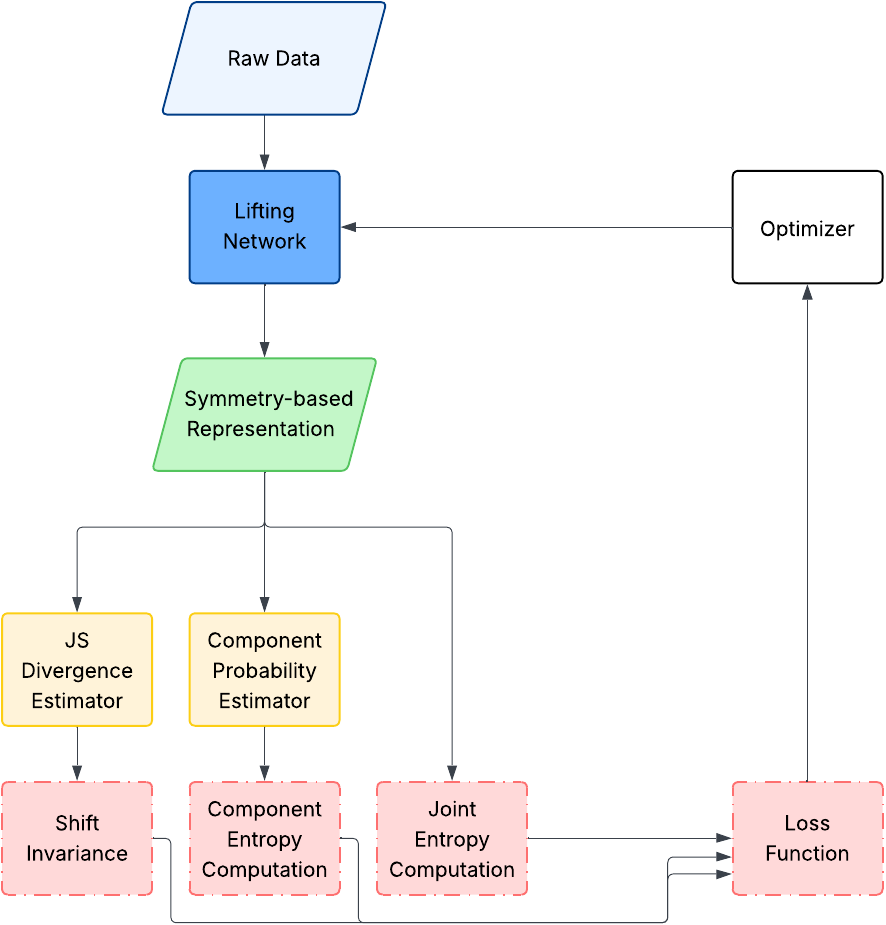}
    \caption{Optimization loop for the lifting network.}
    \label{fig:primary-training-loop}
\end{subfigure}
\hfill
\begin{subfigure}[t]{0.4\linewidth}
    \centering
    \includegraphics[width=\linewidth]{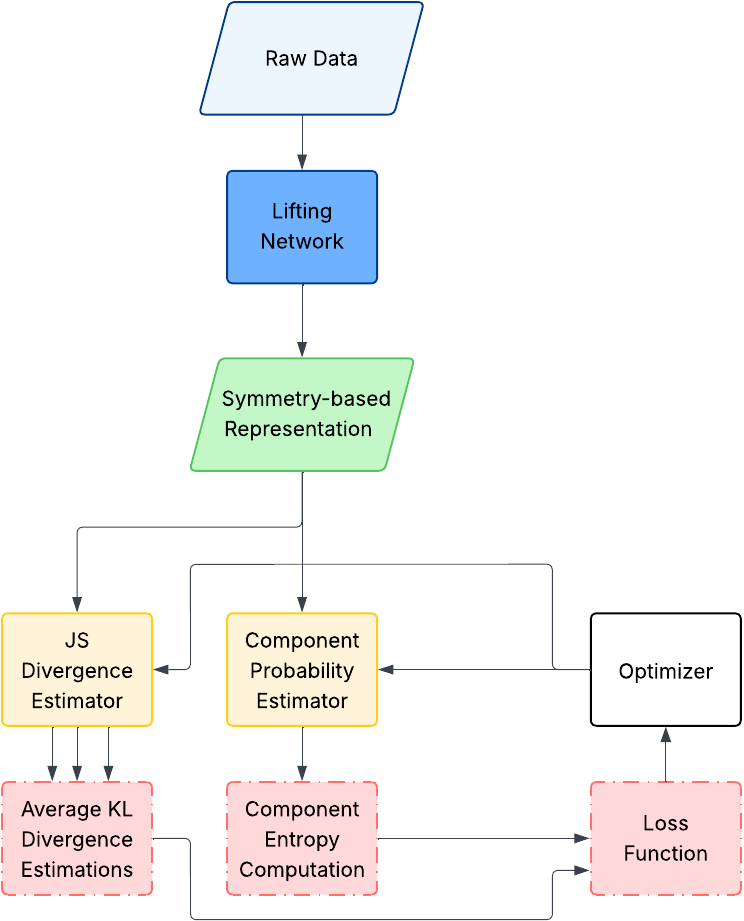}
    \caption{Optimization loop for the auxiliary networks.}
    \label{fig:auxiliary-training-loop}
\end{subfigure}
\hfill
\begin{subfigure}[t]{0.125\linewidth}
    \centering
    \includegraphics[width=\linewidth]{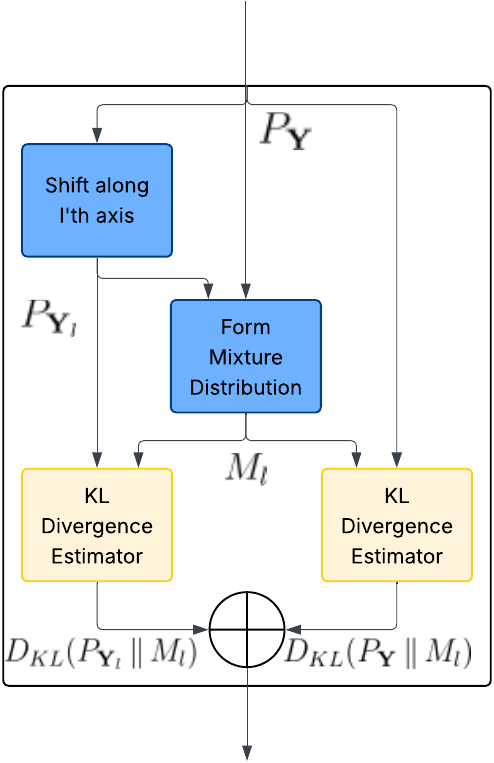}
    \caption{Structure of JS divergence estimator.}
    \label{fig:js-divergence-estimator-structure}
\end{subfigure}
\end{figure}

\subsubsection{KL divergence estimation}
Stationarity loss is given in terms of the KL 
divergence between two distributions. To estimate KL divergence, we use 
the duality formula
\cite{sreekumar2022neuralestimationstatisticaldivergences, 
donsker1975asymptotic} giving the KL divergence $D_{KL}(R || S)$ between two probability distributions $R, S$ as a solution to the optimization 
problem
\begin{equation}\label{eq:kl-div-estimator}
D_{KL}(R || S) = \sup_{\theta \in \Theta}
\bigl\{\mathbb{E}_{\ve{x}\sim R}[f_\theta(\ve{x})] \;-\; \log \mathbb{E}_{\ve{x}\sim S}[\exp{f_\theta(\ve{x})}]\bigr\} \, 
\end{equation}
where \(f_\theta: \mathbb{R}^N \rightarrow \mathbb{R}\) and 
$\theta$ parametrizes all functions \cite{donsker1975asymptotic}. 

We approximate the search over all $f_\theta$ using a neural network, with
$\theta$ describing the network parameters, and using 
\eqref{eq:kl-div-estimator} to define the loss. To approximate the expectation values in \eqref{eq:kl-div-estimator}, we use averaging over
batches. For the network architecture, we use a CNN with additional learnable position embeddings, which turns out to work 
rather efficiently. See Appendix~\ref{app:position-embeddings} for details. 
Since each JS-divergence involves two KL-divergence estimates, we use
two identical networks with separate weights for each term in the loss, 
and train these networks using 
the optimization loop given in Figure~\ref{fig:auxiliary-training-loop}.  

\subsection{Efficient implementation for lifting}\label{app:efficient-lifting-implementation}
To lift efficiently, we use the basis transformation $U$
that simultaneously diagonalizes $L_i$:
\begin{equation}
    y_{\ve{n}}
    = \vs{\phi}^\top e^{-\sum_{i=1}^P {t_{\ve{n}_i} L_i}} \ve{x}
    = \vs{\phi}^\top U e^{-\sum_{i=1}^P {t_{\ve{n}_i} D_i}} U^\dagger \ve{x}
    = \ad{\Tilde{\vs{\phi}}} e^{-\sum_{i=1}^P {t_{\ve{n}_i} D_i}} \Tilde{\ve{x}}
\end{equation}  
where we defined $\Tilde{\vs{\phi}} := U^\dagger \vs{\phi}$ and $\Tilde{\ve{x}} := U^\dagger \ve{x}$ in the 
last expression, and $D_i \in \mathbb{C}^{d \times d}$ are the diagonal eigenvalue matrices for $L_i$. This implementation ensures 
both numerical stability and efficiency by exchanging the matrix exponential with scalar operations.

\subsection{Lie basis parametrizations}\label{app:lie-basis-parametrization}
For Abelian Lie groups, Lie basis elements must share eigenvectors. To ensure this we parametrize all eigenvectors by using a 
single anti-symmetric matrix. Eigenvalues for each Lie basis element are parametrized seperately as real vectors 
$\{ \vs{\alpha}_i \in \mathbb{R}^d: i=1,\dots,P\}$.

To parametrize the Lie basis eigenvectors, we use a single learnable anti-symmetric matrix \(A \in \mathbb{R}^{d \times d}\) and 
obtain the eigenvectors $U \in \mathbb{C}^{d \times d}$ by diagonalization. Eigenvalues $\{\vs{\phi}_i\}$ for each basis are parametrized 
as seperate purely-imaginary vectors, and then exponentiated for obtaining eigenvalues of different group elements.  

However, this parametrization require us to ensure that each Lie basis $L_i := U e^{\left(i \vs{\phi}_i \right)}$, is an orthogonal
matrix. This requires ensuring that seperatey parametrized eigenvalues and their corresponding eigenvectors form conjugate 
pairs. 

To ensure this, we define a permutation matrix $P_{ij} := \ve{u}_i^\top \ve{u}_j$, which shuffles elements of any vector 
such that each component with index $j$ is mapped to the index satisfying $\ve{u}_i = \ve{u}_j^*$. Then we get properly ordered conjugate 
eigenvalue phases $\vs{\phi}_i$ by applying the following operation
\begin{equation}
    \vs{\phi}_{i} = \vs{\alpha}_{i} - P \vs{\alpha}_{i}
\end{equation}
where \(\vs{\alpha}_{i}\) is the parameterization for the \(i\)'th Lie basis.

\subsection{Component-wise entropy estimation}\label{app:marginal-entropy-estimation}
We estimate component-wise entropies using trainable Gaussian mixture models for each component, following \cite{pichler2022differential}. 
This approach models the probability distribution \(\hat{p}\) for each component of y-representation as:
\begin{equation}
    \hat{p}_\ve{n}(x) = \sum_{m=1}^M \frac{w_{\ve{n}; m}}{\sigma_{\ve{n}; m} \sqrt{2\pi}} 
    \exp\left(-\frac{(x - \mu_{\ve{n}; m})^2}{2\sigma_{\ve{n}; m}^2}\right)
\end{equation}
where \(\ve{n} = (n_1, n_2, \dots, n_k)\), with mixture weights satisfying \(\sum_m w_{\ve{n}; m} = 1\) and \(w_{\ve{n}; m} \geq 0\). 
We use \(M=4\) Gaussian components per mixture.

The component-wise entropy \(h_{\mathbf{n}}\) is computed via:
\[
h_{\mathbf{n}} = -\mathbb{E}_{y \sim P} \left[ \log \hat{p}_{\mathbf{n}}(y_{\mathbf{n}}) \right].
\]

For training probability estimators, we minimize the average entropy \(\frac{1}{d} \sum_{\mathbf{n}} h_{\mathbf{n}}\)
across all components where \(d\) denotes the total number of components.

\subsection{Joint entropy estimation}\label{app:low-rank-entropy-estimation}
To estimate the multidimensional differential entropy of the lifted representation $\mathbf{y}$, 
we use a multivariate Gaussian approximation. For a multivariate Gaussian with covariance matrix 
$\mathbf{C}$ whose eigenvalues are $\lambda_l$, the total entropy $h(\mathbf{y})$ is given, up to a constant shift,
by $h(\mathbf{y}) = \sum_l h_l$ where $h_l = \log(\lambda_{l})$.

For each batch, we flatten the lifted representation and compute the sample covariance
matrix $c_{pr} = \operatorname{cov}(y_{\mathbf{i}(p)}, y_{\mathbf{j}(r)})$ where \(p\) and \(r\) 
correponds to flattened indexes. Then we apply eigendecomposition and sort its eigenvalues in descending order. 
Rank-$k$ approximation to the entropy as a weighted average of the per-component contributions to entropy 
is calculated as follows
\begin{equation} \label{low-rank-entropy}
    \bar{h}_k \triangleq \frac{\sum_{l=1}^d w_{ki}\log(\lambda_l)}{\sum_{l=1}^d w_{kl}}
\end{equation}
where the weights provide a soft thresholding at $l=k$. 

We compute weights using the sigmoid function via
\begin{equation}\label{entropy-rank-weights}
    w_{kl} = \frac{1}{e^{\alpha (l - k)} + 1} 
\end{equation}
with $\alpha$ a hyperparameter determining the smoothness of the transition. Our experiments have shown 
consistent results for various values of $\alpha>1$ (we chose $\alpha=3.3$).

Due to the normalization, \eqref{low-rank-entropy} should be thought of as a per-rank version of the 
low-rank entropy. In particular, when combining $\bar{h}_k$ with the marginal entropies of $y_\mathbf{i}$
during the computation of the total correlation loss, it is more appropriate to combine the former with 
the average marginal entropy rather than the total marginal entropy.

\subsection{JS-divergence estimators}
\subsubsection{Overview}
In general, we probe shift-invariance by estimating the Jensen-Shannon (JS) divergence between the probability 
distribution of the $y$-representation and its shifted versions along each axis. 

To estimate JS divergence, first we estimate 2 KL divergence term of each group axis, leveraging the dual representation 
which is given in Equation~\ref{eq:kl-div-estimator}. Then use the KL divergences to compute JS divergence by 
a simple identity. This approach requires \(2\) networks for each group dimension. This approach enables us to 
formulate KL-divergence estimation as a highly efficient downstream task. 

We use CNNs due to their stability and convergence speed, while we use position embeddings to improve their expressive capacity. 
For one-parameter 63 dimensional datasets, we also apply a coarse-grained lifted representation, and estimate 
JS-divergence in two scales. Details to the coarse-graining procedure is given in Appendix~\ref{app:coarse-graining}.

\subsubsection{Position embeddings}\label{app:position-embeddings}
The dual representation for KL divergence requires solving the optimization problem in Equation~\ref{eq:kl-div-estimator} 
over the space of all functions \(f: \mathbb{R}^{d_1 \times d_2 \times \cdots \times d_k} \to \mathbb{R}\). In principle, 
deep multilayer perceptrons (MLPs) appear suitable for this task as they lack inductive biases. However, our experiments 
reveal that MLPs exhibit slow convergence, inducing instabilities in our method.

While Convolutional Neural Networks (CNNs) converge faster and more stably than MLPs, they can only represent translation-equivariant 
functions. To improve expressive capacity while preserving convergence efficiency, we introduce learnable position embeddings. 
The augmented y-representation becomes:
\begin{equation}
    y^{\text{embed}}_{n_1, n_2, \dots, n_k; e} := y_{n_1, n_2, \dots, n_k} + p_{n_1, n_2, \dots, n_k; e}
\end{equation}
where \(p \in \mathbb{R}^{d_1 \times d_2 \times \cdots \times d_k \times d_e}\) denotes the learnable position embedding tensor, 
and \(d_e = 4\) is the embedding dimension used throughout our experiments.

\subsubsection{Coarse-graining}\label{app:coarse-graining}
Experiments indicate that high dimensionality along an axis can cause optimization challenges, such as optimizers becoming trapped in 
local minima. To mitigate this issue, we introduce a coarse-grained version of the $y$-representation and employ an additional 
JS-divergence estimator, running in lower dimensional lifted representation.

The coarse-graining procedure consists of two steps:
\begin{enumerate}
    \item Partition the $y$-representation into patches along each axis.
    \item Randomly select one component within each patch.
\end{enumerate}
This operation is formalized in Equation~\ref{eq:coarse-graining}, where the coarse-grained representation $y^{\text{coarse}}$ is 
indexed by positive integers $(n_1, n_2, \dots, n_k)$. For each sample, we independently generate random integer offsets $\{r_i\}$ 
uniformly distributed in $\{0, 1, \dots, q_i - 1\}$, where $q_i$ denotes the dilation factor along axis $i$ for $i = 1, \dots, k$. 
The indices $(n_1, n_2, \dots, n_k)$ range over $\left(1, 2, \dots, \lfloor d_1/q_1 \rfloor\right) \times \cdots \times 
\left(1, 2, \dots, \lfloor d_k/q_k \rfloor\right)$.

\begin{equation}\label{eq:coarse-graining}
y^{\text{coarse}}_{n_1, n_2, \dots, n_k} := y_{n_1 q_1 + r_1,  n_2 q_2 + r_2,  \dots,  n_k q_k + r_k}
\end{equation}

\subsection{Weighting JS-divergence estimations}
\subsubsection{Scaling properties of JS-divergence}
In our method, we quantify shift-invariance by estimating the Jensen-Shannon (JS) divergence between the probability distribution 
\(P\) of the y-representation and its shifted variants along each axis. However, when axes have substantially different 
dimensionalities as a result of coarse graining, averaging JS divergences without accounting for dimensional effects could degrade performance.

To address this, we weight JS divergence estimates based on their scaling behavior with respect to dimensionality. We assume 
the y-representation is the discrete version of a rectangular region of a Lie manifold which is reflected in data distribution. 
Under this assumption, increasing the dimensionality along an axis corresponds to nothing more than increasing resolution, 
allowing us to estimate the scaling behavior of JS divergence. 

Formally, we can parametrize \(P\) with real parameters \(\theta\) such that a small shift in y-representation corresponds to 
a perturbation \(\delta \theta\) in parameter space. Expanding the JS divergence in a Taylor series to second order at 
\(\theta_0\) we have
\begin{align}
    \text{JS}\big(P(\theta_0) \parallel P(\theta_0 + \delta \theta)\big) 
    &= \text{JS}\big(P(\theta_0) \parallel P(\theta_0)\big) \\
    &\quad + \sum_i \frac{\partial}{\partial \theta_i}\text{JS}\big(P(\theta_0)\parallel P(\theta)\big)\Big|_{\theta=\theta_0} \delta \theta_i \\
    &\quad + \frac{1}{2}\sum_{i,j} \frac{\partial^2}{\partial \theta_i \partial \theta_j}\text{JS}\big(P(\theta_0)\parallel P(\theta)\big)\Big|_{\theta=\theta_0} \delta \theta_i \delta \theta_j \\
    &\quad + \mathcal{O}(\|\delta \theta\|^3).
\end{align}
Since \(\text{JS}\big(P(\theta_0) \parallel P(\theta_0)\big) = 0\) and the first derivative vanishes at the global minimum 
\(\theta = \theta_0\), we obtain the second-order approximation:
\begin{equation}\label{eq:js-expansion}
    \text{JS}\big(P(\theta_0) \parallel P(\theta_0 + \delta \theta)\big) 
    \approx \frac{1}{2}\sum_{i,j} \frac{\partial^2}{\partial \theta_i \partial \theta_j}\text{JS}\big(P(\theta_0)\parallel 
    P(\theta)\big)\Big|_{\theta=\theta_0} \delta \theta_i \delta \theta_j
\end{equation}

Under Equation~\ref{eq:js-expansion}, scaling the shift magnitude by \(s > 0\) (which scales \(\delta \theta \to s\delta \theta\) 
for small continuous shifts) yields:
\begin{equation}\label{eq:js-scaling-translation}
\text{JS}\big(P(\theta_0) \parallel P(\theta_0 + s\delta \theta)\big) \approx s^2 \cdot \text{JS}\big(P(\theta_0) \parallel 
P(\theta_0 + \delta \theta)\big),
\end{equation}
demonstrating quadratic scaling of JS divergence with shift magnitude. 

Since we estimate \(\text{JS}(P \parallel P_l)\) for each axis \(l\) where \(P_l\) is the one-component shifted version of \(P\) 
along \(l\)'th axes, and we assume that increasing dimensionality merely changes the resolution, magnitude of one-component shift 
should be proportional to \(1 / d_l\). Then from Equation~\ref{eq:js-scaling-translation}, we conclude that 
\(\text{JS}(P \parallel P_l) \propto 1/d_l^2\) while everything else kept the same. 

To maintain consistency with the use of per-dimension entropy in Total Correlation and Infomax losses, we further divide by 
JS-divergences by the total dimension \(d := \prod_{i=1}^k d_i\). This weighting is on the same footing with other loss term 
which is justified by considering the mutual information form of JS divergence
\begin{equation}
    \text{JS}(P \parallel P_l) = h(P) - h\left(P \;\middle|\; \frac{P + P_l}{2}\right)
\end{equation}
where the per-dimension expression 
\(\frac{1}{d}\text{JS}(P \parallel P_l) = \frac{1}{d} h(P) - \frac{1}{d} h\left(P \;\middle|\; \frac{P + P_l}{2}\right) \) includes 
joint and conditional entropies per dimension.

Thus, we assign the weight \(w_l := d_l^2 / d\) to the JS divergence along axis \(l\). For coarse-grained representations, 
we exchange \(d_l\) with \(d_l/q_l\) and \(d\) with \(\prod_{i=1}^k d_i/q_i\) which is actually dimensions after the coarse-graining 
took place. This scheme ensures consistent hyperparameter settings across experiments when the stated assumptions hold.

Consequently, we rescale the JS-divergence terms $D_{\mathrm{JS}} \left( P_{\rvector{Y}} || P_{\rvector{Y}^{(l)}} \right) \to 
w_l D_{\mathrm{JS}} \left( P_{\rvector{Y}} || P_{\rvector{Y}^{(l)}} \right)$.

\subsection{Controlling the rank of the joint entropy estimator}
To help facilitate efficient optimization by first fitting
to the gross features of the data, and refining over time, 
we use a time-dependent rank parameter $k$ for the entropy estimator 
\eqref{entropy-rank-weights}. 
To adjust $k$ during training, we use a normalized notion of training time 
$t_n$ measuring the ``amount of gradient flow'' via
\begin{equation}
  t_n = \frac{\sum_{s=1}^{n} \text{lr}(s)}{\sum_{s=1}^{T} \text{lr}(s)}
\end{equation}
where $\text{lr(s)}$ is the learning rate used at the training step (batch) $s$, 
and $n$ and $T$ are the current training steep, and the total 
number of training steps, respectively. 
We control the rank $k$ of the low-rank entropy estimator by setting
$k = \text{ceil}(d \times t_n)$
so that by the end of the training, the rank is at $d$.

\subsection{Architectures of KL-divergence estimators}
\begin{table}
    \centering
    \caption{Detailed architecture specifications for the KL-divergence estimators across all experimental regimes. The table distinguishes 
    between 1D signal benchmarks (top) and 2D image benchmarks (bottom). 
\textbf{Notation:} 
\textit{Lyr}: Layer sequence index; 
\textit{Filters}: Number of output feature maps; 
\textit{Kernel}: Spatial receptor size ($k$ for 1D, $k \times k$ for 2D); 
\textit{Pool}: Downsampling mechanism (\textit{Stoch}: Stochastic Pooling, \textit{None}: Identity); 
\textit{Pad/BC}: Padding strategy and boundary conditions (\textit{Cyc}: Cyclic/Circular padding for 
periodic domains, \textit{Zero}: Standard zero-padding). 
\textbf{Components:} 'Main Net' refers to the primary feature extractor. 'Input Net' denotes a learnable linear 
embedding layer ($d_{\text{emb}}=4$) applied prior to the main network for high-dimensional or scrambled inputs (e.g., Bittensor). 
\textbf{Global Settings:} Unless otherwise noted, all convolutional layers employ ELU activations and Spectral Normalization (max norm $10$).}
    \label{tab:arch-details}
    \vspace{0.5em}
    
    \resizebox{0.70\textwidth}{!}{%
    
    \small 
    \begin{tabular}{l l c c c c c c}
        \toprule
        \textbf{Experiment} & \textbf{Network} & \textbf{Lyr} & \textbf{Filters} & \textbf{Kernel} & \textbf{Stride} & \textbf{Pool} & \textbf{Pad/BC} \\
        \midrule
        \multicolumn{8}{c}{\textbf{\textit{1D Domain (Kernels are scalar)}}} \\
        \midrule
        
        \textbf{Shot Noise (63D)} & Main Net & 1 & 4 & 5 & 1 & Stoch & Zero \\
        & & 2 & 4 & 5 & 1 & Stoch & Zero \\
        & & 3 & 6 & 3 & 1 & Stoch & Zero \\
        & & 4 & 6 & 3 & 1 & Stoch & Zero \\
        & & 5 & 8 & 3 & 1 & Stoch & Zero \\
        \midrule
        
        \multirow{5}{*}{\textbf{LieGAN}} 
        & Network A & 1 & 4 & 5 & 1 & Stoch & Cyclic \\
        & & 2 & 6 & 5 & 1 & Stoch & Cyclic \\
        & & 3 & 8 & 3 & 1 & Stoch & Cyclic \\
        \cmidrule{2-8}
        & Network B & 1 & 6 & 5 & 1 & Stoch & Cyclic \\
        & & 2 & 8 & 3 & 1 & Stoch & Cyclic \\
        \midrule

        \multirow{7}{*}{\textbf{Ising Model}} 
        & Network A & 1 & 4 & 5 & \textbf{2} & \textit{None} & Zero \\
        & & 2 & 6 & 5 & 1 & Stoch & Zero \\
        & & 3 & 6 & 3 & 1 & Stoch & Zero \\
        & & 4 & 8 & 3 & 1 & Stoch & Zero \\
        \cmidrule{2-8}
        & Network B & 1 & 4 & 5 & \textbf{2} & \textit{None} & Zero \\
        & & 2 & 6 & 3 & 1 & Stoch & Zero \\
        & & 3 & 8 & 3 & 1 & Stoch & Zero \\
        \midrule

        \multirow{4}{*}{\textbf{Visual Coding}} 
        & Main Net & 1 & 4 & 7 & 1 & Stoch & Zero \\
        & & 2 & 6 & 5 & 1 & Stoch & Zero \\
        & & 3 & 8 & 3 & 1 & Stoch & Zero \\
        & & 4 & 8 & 3 & 1 & Stoch & Zero \\

        \midrule
        \multicolumn{8}{c}{\textbf{\textit{2D Domain (Kernels are $k \times k$)}}} \\
        \midrule

        \textbf{MNIST 15$\times$15} & Main Net & 1 & 4 & $5\times5$ & 1 & Stoch & Zero \\
        (Permuted) & & 2 & 6 & $3\times3$ & 1 & Stoch & Zero \\
        & & 3 & 8 & $3\times3$ & 1 & Stoch & Zero \\
        \midrule

        \multirow{4}{*}{\shortstack[l]{\textbf{MNIST 15$\times$15} \\ (Bittensor)}} 
        & Main Net & \multicolumn{6}{c}{\textit{Same as Permuted Main Net}} \\
        \cmidrule{2-8}
        & Input Net & 1 & 4 & $3\times3$ & 1 & \textit{None} & Zero \\
        \midrule

        \multirow{5}{*}{\textbf{MNIST 27$\times$27}} 
        & Main Net & 1 & 4 & $7\times7$ & 1 & Stoch & Zero \\
        & & 2 & 6 & $5\times5$ & 1 & Stoch & Zero \\
        & & 3 & 8 & $3\times3$ & 1 & Stoch & Zero \\
        & & 4 & 8 & $3\times3$ & 1 & Stoch & Zero \\
        \cmidrule{2-8}
        & Input Net & 1 & 4 & $3\times3$ & 1 & \textit{None} & Zero \\
        
        \bottomrule
    \end{tabular}
    }
\end{table}

\begin{figure}[ht]
    \centering
    \begin{tikzpicture}[
        node distance=1.2cm,
        >=stealth,
        block/.style={
            rectangle, 
            draw=black!70, 
            thick, 
            fill=gray!5, 
            minimum width=1.5cm, 
            minimum height=2.5em,
            rounded corners=2pt
        },
        arrow/.style={->, thick, black!80},
        label/.style={font=\small\sffamily, text=black!70}
    ]

    \node (input) [font=\bfseries] {Input};
    
    \node (layer1) [block, right=0.8cm of input] {Layer 1};
    \node (layer2) [block, right=0.8cm of layer1] {Layer 2};
    \node (dots)   [right=0.5cm of layer2, font=\Large] {$\dots$};
    \node (layerN) [block, right=0.5cm of dots] {Layer $N$};
    
    \node (output) [font=\bfseries, right=0.8cm of layerN] {Output};

    \draw[arrow] (input) -- (layer1);
    \draw[arrow] (layer1) -- (layer2);
    \draw[arrow] (layer2) -- (dots);
    \draw[arrow] (dots) -- (layerN);
    \draw[arrow] (layerN) -- (output);

    \node (detail) [
        rectangle, 
        draw=blue!50, 
        fill=blue!5, 
        dashed, 
        thick, 
        below=1.5cm of layer2, 
        minimum width=5cm, 
        minimum height=1.5cm,
        align=center
    ] {};
    
    \draw[dashed, blue!40] (layer2.south west) -- (detail.north west);
    \draw[dashed, blue!40] (layer2.south east) -- (detail.north east);

    \node (conv) [block, fill=white, draw=blue!50, scale=0.8, below left=0.15cm and 1.5cm of detail.center] {Conv};
    \node (act)  [block, fill=white, draw=orange!50, scale=0.8, right=0.5cm of conv] {ELU};
    \node (pool) [block, fill=white, draw=green!50, scale=0.8, right=0.5cm of act] {Pool};

    \draw[arrow, blue!50] (conv) -- (act);
    \draw[arrow, blue!50] (act) -- (pool);
    
    \node [above=0.1cm of detail.north, font=\footnotesize\bfseries, text=blue!60] {Generic Block Structure};

    \end{tikzpicture}
    \caption{General architecture of the proposed estimators. Each network consists of a sequence of processing blocks. 
    As illustrated in the breakout view, each block typically comprises a convolution, a non-linear activation (ELU), and a 
    pooling operation (Stochastic/Max). Specific kernel sizes and filter counts for each estimator are detailed in Table~\ref{tab:arch-details}.}
    \label{fig:generic-arch}
\end{figure}
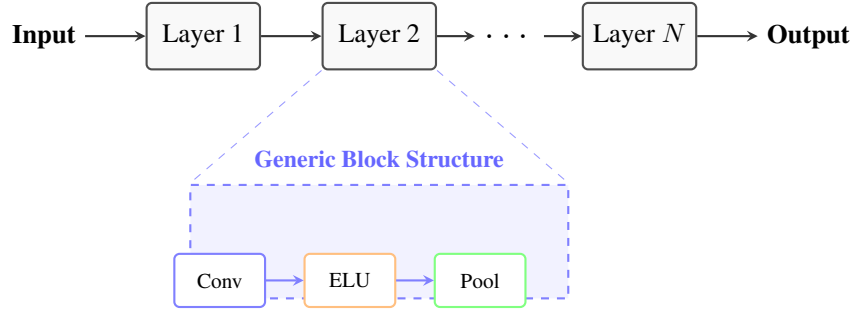

Our networks are standard CNN downstream networks, with the additional position embedding at input, as shown in Figure~\ref{fig:generic-arch}. At each 
layer, we either use strides or pooling layer to downsample the features along spatial axis. 
Additionally, we constrain the spectral norm of each network to $10.$, by applying a layerwise spectral norm 
limitation which translates to $10^{\frac{1}{\#\text{layers}}}$ for each layer.

To reduce the computational costs, we use strides in the first few layer of the each network, since these are the most 
compute intensive parts. In the remaining layers we use stochastic pooling ~\cite{zeiler2013stochastic} 
to battle with the curse of dimensionality. For challenging settings, we use more than one js-divergence estimators, while one of them running over a coarser grid. 

For the specifics of the network architecture, refer to 
Table~\ref{tab:arch-details}.
\section{Training details}\label{app:training-details}
All models are optimized using the Adam algorithm. We employ the default hyperparameters: $\beta_1 = 0.9$, $\beta_2 = 0.999$, and $\epsilon = 10^{-7}$. The learning rate followed an exponential decay schedule, decreasing smoothly from the initial value to the final value listed in Table~\ref{tab:training-configurations}.  

\begin{table}[ht]
    \centering
    \caption{Experiment details.}
    \label{tab:training-configurations}
    \vspace{0.5em}
    \resizebox{\linewidth}{!}{%
    \begin{tabular}{lcccccccccccc}
        \toprule
        \thead{Experiment} & \thead{Dataset\\size} & \thead{Duration} & \thead{Epochs} & \thead{Eigendecomposition\\algorithm} 
        & \thead{Batch size} & \thead{Primary\\initial lr} & \thead{Auxiliary\\initial lr} & \thead{$\frac{\text{Lr}_\text{initial}}{\text{Lr}_\text{final}}$}
        & \thead{Invariance\\weight} & \thead{Resolution\\weight} & \thead{Infomax\\weight} & \thead{$\frac{\text{Noise(std)}}{\text{Signal(std)}}$} \\
        \midrule
        \makecell{GSN\\63 timesteps} & 500k & 11 hrs & 10000 & SVD & 5000 & $10^{-4}$ & $10^{-3}$ & 0.1 & 1.0 & 1.0 & 0.75 & 0.05 \\
        \hline
        \makecell{ISING\\33 timesteps} & 500k & 7.5 hrs & 10000 & SVD & 5000 & $10^{-4}$ & $10^{-3}$ & 0.1 & 1.0 & 1.0 & 0.75 & 0.05 \\
        \hline
        \makecell{Visual Coding \\110 neurons} & k & 0.25 hrs & 250 & SVD & 2500 & $10^{-4}$ & $10^{-3}$ & 0.1 & 2.5 & 0.2 & 0.2 & 0.0 \\
        \hline
        \makecell{GSN\\15 timesteps} & 150k & 1 hrs & 1500 & SVD & 500 & $5\times10^{-4}$ & $5\times10^{-3}$ & 0.1 & 1.5 & 1.0 & 0.75 & 0.05 \\
        \hline
        \makecell{MNIST 15x15\\permuted\\translation} & 500k & 41 hrs & 15000 & SVD & 2500 & $10^{-4}$ & $10^{-3}$ & 0.1 & 1.0 & 0.2 & 0.15 & 0.00 \\
        \hline
       \makecell{MNIST 27x27\\permuted\\translation} & 500k & 132 hrs & 35000 & EIG & 2500 & $10^{-4}$ & $10^{-3}$ & 0.1 & 1.0 & 0.2 & 0.15 & 0.00 \\
       \hline
       \makecell{MNIST 15x15\\bit-scrambled} & 500k & 8 hrs & 2000 (Early stop) & SVD & 2500 & $10^{-4}$ & $10^{-3}$ & 0.1 & 1.0 & 0.2 & 0.15 & 0.00 \\
        \hline
       \makecell{MNIST 27x27\\bit-scrambled} & 2500 & 28 hrs & 7500 (Early stop) & EIG & 2500& $10^{-4}$ & $10^{-3}$ & 0.1 & 1.0 & 0.2 & 0.15 & 0.00 \\
       \bottomrule
   \end{tabular}
   }
\end{table}
\section{Synthetic datasets}\label{appsec:synthetic-dataset}
\subsection{Ising model}\label{appsubsec:ising}
To generate spin signals, first initialize a state vector of discrete variables with values in $\{-1, +1\}$, 
where each component represents a distinct node. We then perform $10$ iterations of stochastic updates based on nearest-neighbor 
interactions where strength is controlled by the inverse temperature $\beta$, 
which we sample uniformly from $\mathcal{U}[1.0, 5.0]$ for each dataset sample separately.

After generating the stochastic samples following this procedure, we obscure the domain by applying a fixed random linear 
transformation with singular values bounded in $[0.5, 2.0]$. This transformation acts as an unknown, imperfect sensor response, 
effectively turning the task into a blind inverse problem. In this case, 
to enable the model to compensate for linear transformations more general than orthogonal operators, 
we use a learnable linear embedding $E: \mathcal{X} \to \mathcal{X}_A$ rather than fixed zero-padding.

\subsection{Generalized Shot Noise process}\label{appsubsec:gsn-process}
The synthetic samples were generated through the following procedure:
\begin{itemize}
    \item For each sample $\mathbf{x}$, first determine the number of constituent signals by drawing 
    $P \sim \mathcal{U}\{0, 1, \dots, P_{\text{max}}\}$ with $P_{\text{max}} = 10$.
    \item For each signal $n = 1$ to $P$, draw shape parameters $\theta_n \sim \mathcal{U}(\Theta)$ and 
    draw a translation $\tau_n \sim \mathcal{U}[-L, L]$.
    \item Construct the discrete-time signal via superposition: $
    \mathbf{x}[i] = \sum_{n=1}^{P} h_{\theta_n}(t_i - \tau_n) \quad \text{for } i = 1, \ldots, 63,
    $
    where $t_i$ defines a discrete time grid restricted between $[-L/2, L/2]$ to prevent boundary effects 
    distorting translation invariance.
    \item Finally, we add Gaussian noise to the sample (with \(\sigma = 0.05\)).
\end{itemize}

\begin{figure}[!h]
    \centering
    \begin{subfigure}[t]{0.475\textwidth}
        \centering
        \includegraphics[width=\linewidth]{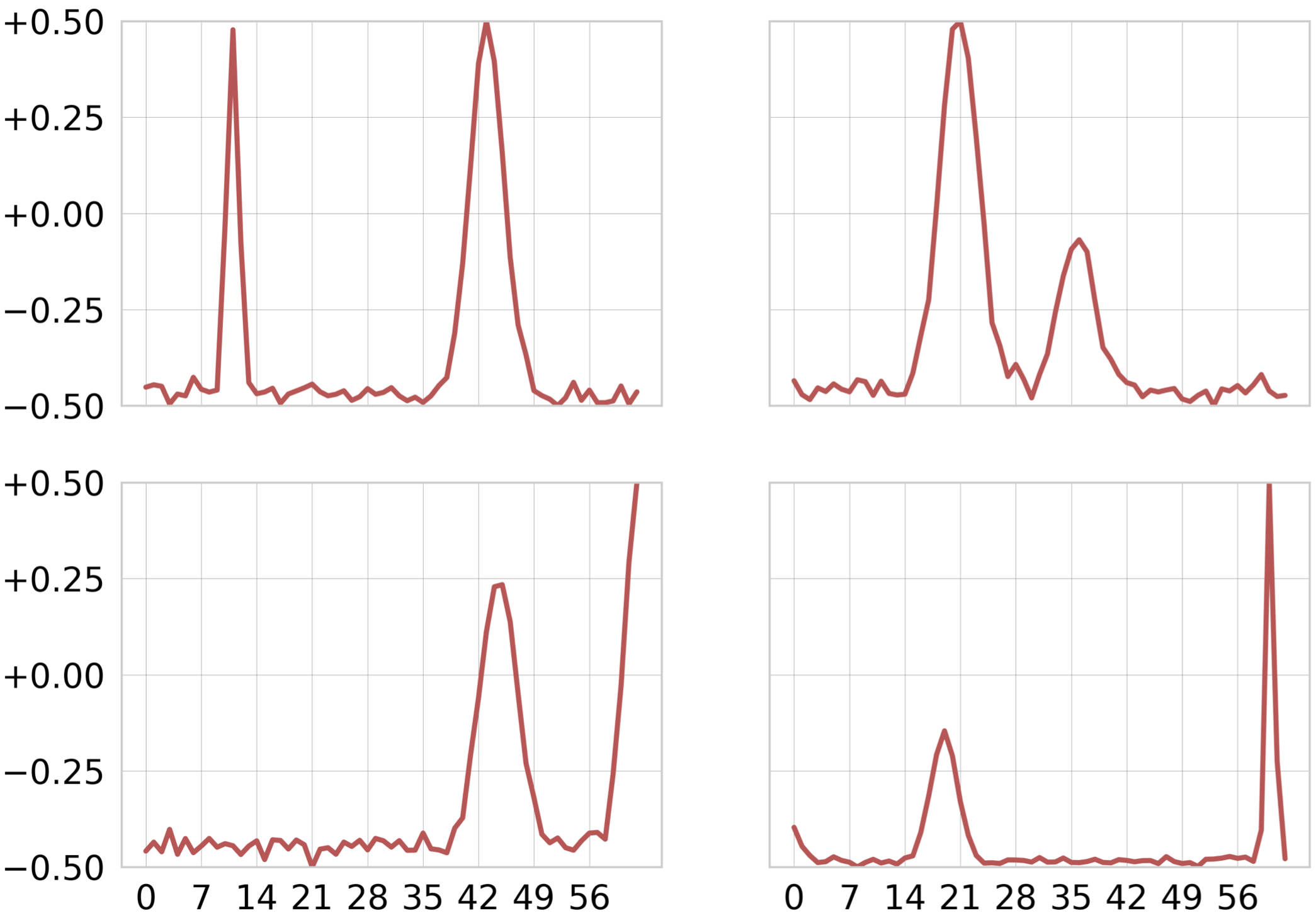}
        \caption{Gaussian Samples of Translation-Invariant Data Distribution}
        \label{fig:samples-1d-translation}
    \end{subfigure}
    \hfill
    \begin{subfigure}[t]{0.475\textwidth}
        \centering
        \includegraphics[width=\linewidth]{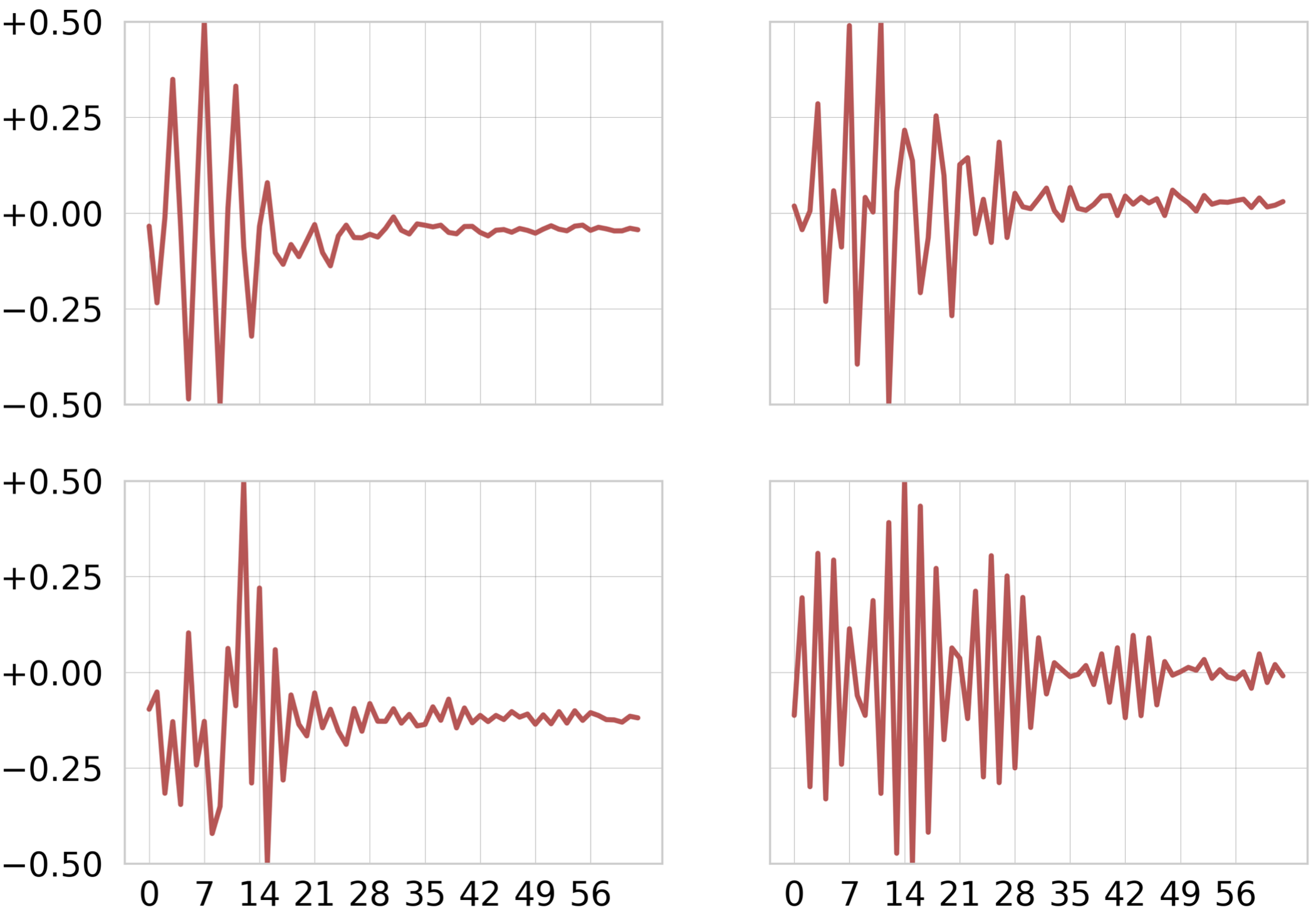} 
        \caption{Legendre Samples of Frequency-Shift Invariant Data Distribution}
        \label{fig:samples-1d-frequency-shift}
    \end{subfigure}
    \caption{Invariant Data Distributions: Gaussian vs. Legendre}
    \label{fig:samples-1d}
\end{figure}

\subsubsection{Data-generation procedure}
\begin{table}
\centering
\scriptsize
\caption{The synthetic datasets prepared for experiments.
The parameters for each basis signal are sampled from a uniform 
distribution with the indicated ranges.}
\label{tab:synthetic-datasets}
\begin{tabular}{lccccc}
\toprule
\makecell{\textbf{Signal} \\ \textbf{Type}} & 
\makecell{\textbf{Input} \\ \textbf{dimensions}} & 
\textbf{Transform} & 
\makecell{\textbf{Amplitude} \\ \textbf{range}} &  
\makecell{\textbf{Scale} \\ \textbf{range}} \\
\midrule
Gaussian & 63 & Identity & [0.5, 1.5) & [0.5, 2.5) \\
Legendre (l=2-3, m=1) & 63 & DST-I & [0.5, 1.5) & [6.0, 15.0) \\
\bottomrule
\end{tabular}
\end{table}

\subsubsection{Gaussian signals} $f_{gaussian}(z;\mathcal{A}, 
\mu, \sigma)$ are parametrized by amplitude \(\mathcal{A}\), center \(\mu\), and width \(\sigma\). 
The input $z$ is an integer ranging from $-32$ to $32$, labeling the components of the raw 
sample vectors. We sample the center $\mu$ uniformly from the extended (tripled) range $-97$ to $97$, 
and then crop the resulting  signals to the $z$ range of $-32$ to $32$ to allow for the possibility of signals that contain 
only a tail of a Gaussian. $A$ and $\sigma$ are sampled from the ranges given in Table 
\ref{tab:synthetic-datasets}.

\subsubsection{Legendre Signals} are given in terms of the associated Legendre polynomials
and give localized waveforms that can change signs. The relevant parameters are center 
\(c\), scale \(s\), amplitude \(\mathcal{A}\), and the orders $l$, $m$:
$f^{(l,m)}_{legendre}(z;\mathcal{A}, c, s) =
\mathcal{A}P_l^m\left(\cos\left(\frac{z-c}{s}\right)\right)$.
We crop these signals to the range 
$|x-c|/s \le \pi$, i.e., set the values outside this
range to zero. 

Once again, $z$ becomes the discrete dimension index, ranging from $-32$ to $32$.
For the $l,m$ parameters, we use $l=2, m=1$ and $l=3, m=1$, with equal probability 
for each sample. We sample the centers as in the Gaussian case, and the sampling of 
the other parameters is described in Table \ref{tab:synthetic-datasets}.

This procedure can give us any symmetry that is related to component translations via a similarity transformation. See Figures 
\ref{fig:samples-1d-translation} and \ref{fig:samples-1d-frequency-shift}, which involve samples from datasets with 
different kinds of symmetries. 
\section{Experiment details and results}\label{appsec:experiment-details-and-result}
\subsection{Neural coding experiment details}~\label{appsubsec:neural-experiment-details}
To isolate orientation geometry from confounding factors such as contrast and speed, we restrict our 
analysis to a fixed spatial frequency ($0.04$ cpd). Raw spike trains are converted to firing-rate estimates 
using a $250\,\mathrm{ms}$ integration window, with an $80\,\mathrm{ms}$ latency correction to account for cortical processing time. 
The selected recording session contains responses from $N=110$ well-isolated neurons, yielding a $110$-dimensional sample space.

The dataset contains approximately $1100$ trials, leading to substantial data sparsity.
To mitigate this limitation, we apply stochastic mixing augmentation~\cite{zhang2017mixup}, 
which involves randomly superposing input samples. Our framework naturally supports this procedure due to its ability to model intransitive datasets, allowing 
mixed samples to serve as valid realizations of the underlying data-generating process. We set the model output dimensionality to $30$ and restricted the output rank to $3$ to prevent overfitting to noise.

\subsection{Bit scrambling experiments}
\begin{figure}[h]
    \centering
    \begin{subfigure}[b]{0.6\textwidth}
        \centering
        \includegraphics[width=\linewidth]{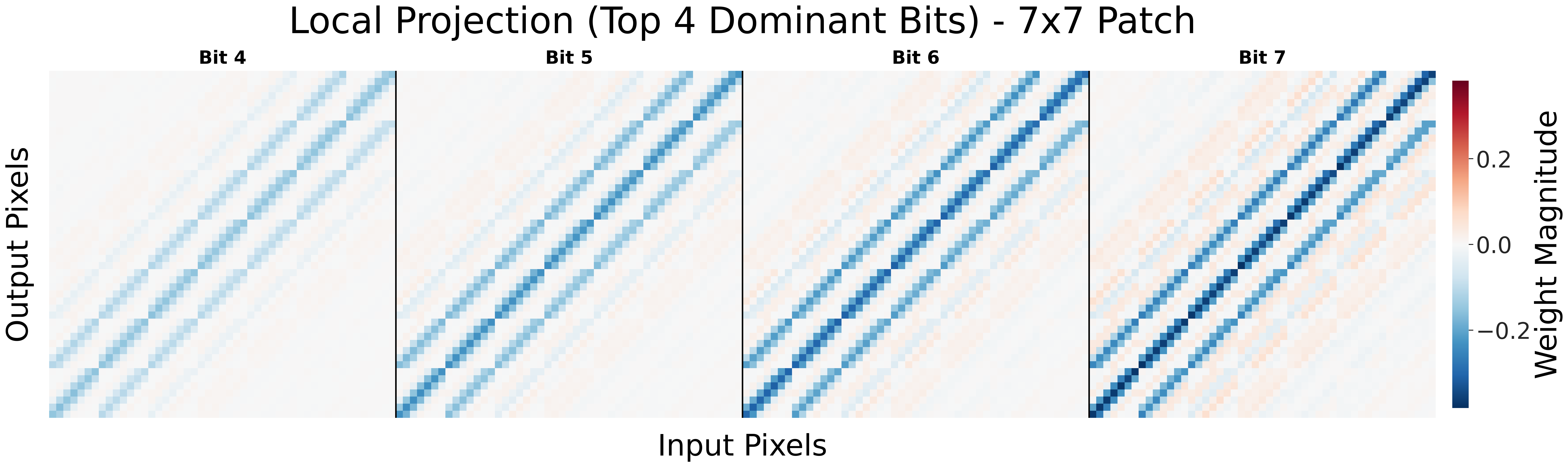} 
        \caption{\textbf{Learned bit-to-byte aggregation.}}
        \label{fig:from_bits_to_images_lifting_matrix_structure}
    \end{subfigure}
    \hfill
    \begin{subfigure}[b]{0.376\textwidth}
        \centering
        \includegraphics[width=\linewidth]{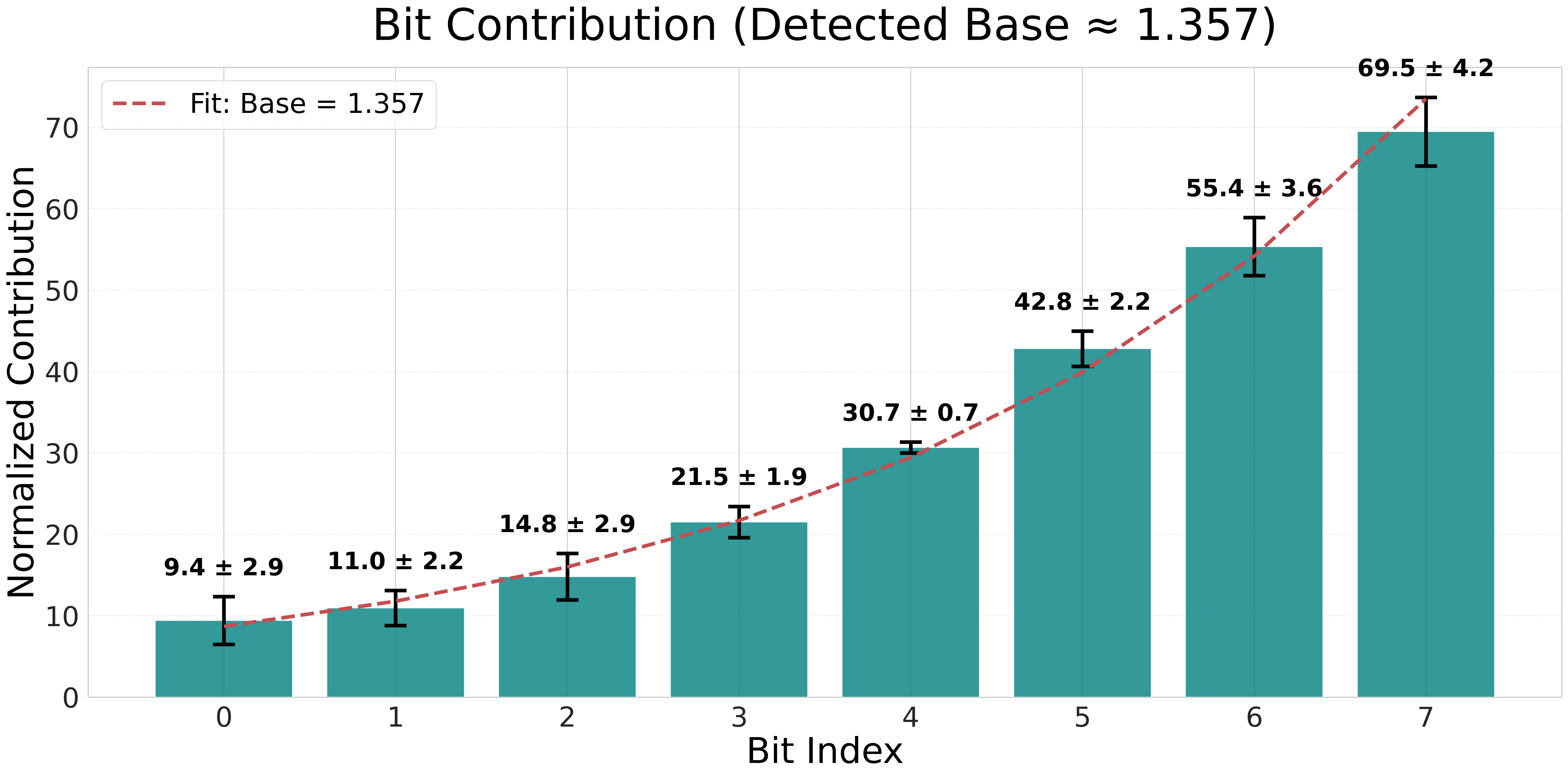}
        \caption{\textbf{Bit weights.}}
        \label{fig:bit_weights}
    \end{subfigure}
    
    \caption{\textbf{Lifting map and bitstream recomposition.} 
    \textbf{(a)} The flattened lifting map displays parallel diagonal lines consistently 
    across each bit channel. Since vertical and horizontal pixel shifts correspond to index increments 
    of 7 and 1 respectively for flattened $7 \times 7$ patch, this pattern forms a banded Toeplitz structure 
    functioning as a local averaging operator with a receptive field of $\approx$ 3 pixels.
    \textbf{(b)} The bit weight distribution shows high stability (low standard deviation) across pixels. 
    The model prioritizes MSBs with a roughly exponential decay, effectively learning an approximate 
    number system with base $1.357$.}
    \label{fig:bits_and_weights_comparison}
\end{figure}

\section{Baselines}\label{baselines}
\subsection{Covariance matching baseline}
\label{baselines-cm}

We include a second-order baseline that attempts to recover the latent domain by matching the covariance spectrum of the observations to the spectrum of a reference graph. This baseline does not estimate a sparse graph or use the hidden signals. Instead, it assumes that the latent signal is smooth on the target domain, so low graph-frequency modes should explain more variance than high graph-frequency modes.

Let \(\mathbf X\in\mathbb R^{n\times d}\) denote the observed samples. We first standardize each observed coordinate,
\[
    \widetilde{\mathbf X}_{ij}
    =
    \frac{\mathbf X_{ij}-\mu_j}{\sigma_j+\epsilon},
\]
and compute the regularized empirical covariance
\[
    \mathbf S
    =
    \operatorname{Cov}(\widetilde{\mathbf X})+\lambda \mathbf I,
\]
where \(\lambda\) is a small diagonal loading term. We then compute the eigendecomposition
\[
    \mathbf S = \mathbf U_{\mathrm{obs}}\boldsymbol\Lambda_{\mathrm{obs}}
    \mathbf U_{\mathrm{obs}}^\top .
\]

For the waveform experiments, the reference domain is a one-dimensional path graph with adjacency matrix \(\mathbf A_{\mathrm{path}}\). For MNIST, the reference domain is the \(H\times W\) four-neighbor image grid with adjacency matrix \(\mathbf A_{\mathrm{grid}}\). In both cases, we write the reference adjacency as
\[
    \mathbf A = \mathbf U_{\mathrm{ref}}\boldsymbol\Lambda_{\mathrm{ref}}
    \mathbf U_{\mathrm{ref}}^\top ,
\]
where \(\mathbf A=\mathbf A_{\mathrm{path}}\) for waveform data and
\(\mathbf A=\mathbf A_{\mathrm{grid}}\) for MNIST.

We consider three variants. The \emph{rotate} variant only aligns the empirical covariance eigenbasis with the reference graph eigenbasis:
\[
    \mathbf W
    =
    \mathbf U_{\mathrm{obs}}\mathbf U_{\mathrm{ref}}^\top .
\]
The \emph{whiten} variant whitens the empirical covariance before applying the reference graph basis:
\[
    \mathbf W
    =
    \mathbf U_{\mathrm{obs}}
    \boldsymbol\Lambda_{\mathrm{obs}}^{-1/2}
    \mathbf U_{\mathrm{ref}}^\top .
\]
Finally, the \emph{recolor} variant whitens the observed covariance and then recolors it using a generic smooth reference spectrum:
\[
    \mathbf W
    =
    \mathbf U_{\mathrm{obs}}
    \boldsymbol\Lambda_{\mathrm{obs}}^{-1/2}
    \boldsymbol\Lambda_{\mathrm{smooth}}^{1/2}
    \mathbf U_{\mathrm{ref}}^\top .
\]
For waveform data, \(\boldsymbol\Lambda_{\mathrm{smooth}}\) is defined using a one-dimensional low-pass spectrum
\[
    \lambda_k^{\mathrm{smooth}}
    =
    \left(1+\gamma \omega_k^2\right)^{-p},
    \qquad
    \omega_k \in [0,1],
\]
normalized to have unit mean. For MNIST, the same construction is applied to graph frequencies derived from the grid adjacency spectrum: high adjacency eigenvalues are treated as low graph frequencies, and low adjacency eigenvalues as high graph frequencies. The parameters \(\gamma\) and \(p\) are selected by grid search.

The recovered samples are obtained by applying the learned linear map to standardized observations,
\[
    \widehat{\mathbf F}
    =
    \widetilde{\mathbf X}\mathbf W .
\]
For waveform data, recovery is evaluated up to the natural one-dimensional ambiguities: translation, reflection, and global sign. For MNIST, recovery is evaluated up to square-grid rotations, reflections, and global sign.

\subsection{Graphical lasso baseline}
\label{baselines-glasso}

We also compare against a graphical-lasso-based graph recovery baseline. This baseline is not used as a direct reconstruction method. Instead, it first estimates sparse conditional dependencies from the observed data and then aligns the resulting graph-like operator with the known reference domain graph.

As in Appendix~\ref{baselines-cm}, we first standardize the observations coordinatewise and compute a diagonally loaded empirical covariance matrix \(\mathbf S\). We then apply graphical lasso to estimate a sparse precision matrix,
\[
    \widehat{\boldsymbol\Theta}
    =
    \arg\max_{\boldsymbol\Theta\succ 0}
    \log\det \boldsymbol\Theta
    -
    \operatorname{tr}(\mathbf S\boldsymbol\Theta)
    -
    \alpha \|\boldsymbol\Theta\|_1 ,
\]
where \(\alpha\) controls the sparsity level. We interpret the absolute off-diagonal precision values as a weighted graph estimate,
\[
    \mathbf A_{\mathrm{glasso}}
    =
    \left|\widehat{\boldsymbol\Theta}\right|,
    \qquad
    \operatorname{diag}(\mathbf A_{\mathrm{glasso}})=0,
\]
and threshold small entries to reduce numerical noise.

For waveform experiments, the target graph is the one-dimensional path adjacency \(\mathbf A_{\mathrm{path}}\). We compute
\[
    \mathbf A_{\mathrm{glasso}}
    =
    \mathbf U_g \boldsymbol\Lambda_g \mathbf U_g^\top,
    \qquad
    \mathbf A_{\mathrm{path}}
    =
    \mathbf U_t \boldsymbol\Lambda_t \mathbf U_t^\top ,
\]
and align the two graph eigenspaces. For orthogonal observation models, we use
\[
    \mathbf W
    =
    \mathbf U_g \mathbf U_t^\top .
\]
For dense linear observation models, we allow a scaled spectral alignment,
\[
    \mathbf W
    =
    \mathbf U_g
    \operatorname{diag}
    \left(
    \sqrt{
    \frac{|\lambda_{t,k}|}{|\lambda_{g,k}|+\epsilon}
    }
    \right)
    \mathbf U_t^\top .
\]
The recovered waveform samples are then given by
\[
    \widehat{\mathbf F}
    =
    \widetilde{\mathbf X}\mathbf W .
\]

For MNIST, we use the same spectral-alignment version of the graphical-lasso baseline, replacing the path graph by the \(H\times W\) four-neighbor grid adjacency \(\mathbf A_{\mathrm{grid}}\). Thus,
\[
    \mathbf A_{\mathrm{grid}}
    =
    \mathbf U_t\boldsymbol\Lambda_t\mathbf U_t^\top,
\]
and the learned recovery map is obtained by aligning the eigenbasis of
\(\mathbf A_{\mathrm{glasso}}\) with the eigenbasis of \(\mathbf A_{\mathrm{grid}}\). Although MNIST observations are generated by pixel permutations, this spectral version is allowed to return a dense orthogonal map rather than a strict permutation. This makes it a more flexible graph-learning baseline and avoids relying on a separate combinatorial graph-matching routine.

For both waveform and MNIST data, we grid-search the graphical lasso regularization parameter \(\alpha\) and report the best-performing configuration under the same recovery metric used for the other baselines.

\subsection{UMAP adapted baseline}
\label{baselines-umap}

We include an adapted UMAP baseline to compare against a standard nonlinear manifold-learning approach. Since UMAP produces low-dimensional coordinates rather than an explicit inverse map or symmetry generator, we adapt it differently depending on the experiment.

\paragraph{MNIST topology recovery.}
For the shuffled MNIST experiments, the goal is to recover the latent two-dimensional image topology from permuted pixel observations. We use UMAP inside the topology-recovery pipeline by applying it to the observed coordinate structure and converting the resulting embedding into interpolation weights on a fixed \(H\times W\) grid. Specifically, for each hyperparameter setting, we instantiate a topology learner with method set to UMAP, grid dimensions \(H,W\), a variance threshold \(\tau\), a neighborhood size \(k\), and a fixed smoothing coefficient. The topology learner returns a set of coordinate embeddings and reconstruction weights, which are then used to reconstruct each permuted image through the image reconstructor:
\[
    \widehat{\mathbf f}_i
    =
    \operatorname{Reconstruct}(\mathbf x_i; \mathbf W_{\mathrm{UMAP}}).
\]
Here, \(\mathbf x_i\) is the shuffled input image and \(\mathbf W_{\mathrm{UMAP}}\) denotes the reconstruction weights induced by the learned UMAP topology.

We perform a grid search over the variance threshold and the number of neighbors. In the implementation used for the MNIST baseline, the fixed smoothing coefficient is set to \(3.0\), the variance thresholds are
\[
    \tau \in \{0.05,0.10,0.15\},
\]
and the neighborhood sizes are
\[
    k \in \{3,5,7\}.
\]
For each configuration, we reconstruct a held-out set of images and evaluate recovery using the same alignment-aware image metric used for the other MNIST baselines. Namely, for each reconstructed image, we compute the maximum absolute Pearson correlation with the original image over the eight dihedral symmetries of the square grid:
\[
    r_i
    =
    \max_{g\in D_4}
    \left|
    \operatorname{corr}
    \left(
    \mathbf f_i,\,
    g\cdot \widehat{\mathbf f}_i
    \right)
    \right|.
\]
The reported score is the average of \(r_i\) over the validation images.

\paragraph{Neural manifold alignment.}
For the neural visual-coding experiment, UMAP is used as a direct two-dimensional manifold embedding baseline. Given neural response vectors
\[
    \mathbf x_i \in \mathbb R^{110},
\]
we fit UMAP to obtain two-dimensional points
\[
    \mathbf z_i = \operatorname{UMAP}(\mathbf x_i) \in \mathbb R^2.
\]
Since the latent stimulus variable is the grating orientation, we evaluate whether the learned two-dimensional manifold recovers the circular organization of the stimulus. We first center the embedded points by their empirical mean and extract a polar phase
\[
    \widehat{\theta}_i
    =
    \operatorname{atan2}
    \left(
    z_{i,2}-\bar z_2,\,
    z_{i,1}-\bar z_1
    \right)
    \quad \operatorname{mod} 2\pi .
\]
The true grating orientations are originally defined modulo \(180^\circ\), so they are mapped to the full circle by
\[
    \theta_i
    =
    \frac{2\pi}{180^\circ}
    \theta_i^{\mathrm{deg}} .
\]
We then compute the circular correlation between the recovered phases and the true stimulus phases. To account for the arbitrary orientation of the learned embedding, we evaluate both the original and mirrored phase directions and keep the better score:
\[
    \max
    \left\{
    \rho_{\mathrm{circ}}(\widehat{\theta},\theta),
    \rho_{\mathrm{circ}}(2\pi-\widehat{\theta},\theta)
    \right\}.
\]
For UMAP, we grid-search the neighborhood size
\[
    k\in\{5,15,30,50\}
\]
and the minimum-distance parameter
\[
    \mathrm{min\_dist}\in\{0.01,0.1,0.5\}.
\]
The configuration with the highest absolute circular correlation is reported.

\subsection{Isomap adapted baseline}
\label{baselines-isomap}

We also include an adapted Isomap baseline. Isomap is a nonlinear manifold-learning method based on neighborhood graphs and approximate geodesic distances. As with UMAP, it does not directly provide a symmetry representation or a reconstruction operator. Therefore, we adapt it to the MNIST and neural settings using the same task-specific evaluation pipelines.

\paragraph{MNIST topology recovery.}
For shuffled MNIST, Isomap is used inside the same topology-recovery framework as UMAP. Given permuted image samples, the topology learner is configured with method set to Isomap, grid dimensions \(H,W\), a variance threshold \(\tau\), a neighborhood size \(k\), and a smoothing coefficient. The learned Isomap coordinates are then converted into reconstruction weights on the target image grid. Each shuffled image is reconstructed as
\[
    \widehat{\mathbf f}_i
    =
    \operatorname{Reconstruct}(\mathbf x_i; \mathbf W_{\mathrm{Iso}}),
\]
where \(\mathbf W_{\mathrm{Iso}}\) denotes the reconstruction weights induced by the Isomap-based topology estimate.

We use the same hyperparameter grid as in the UMAP topology baseline:
\[
    \tau \in \{0.05,0.10,0.15\},
    \qquad
    k \in \{3,5,7\},
\]
with smoothing coefficient fixed to \(3.0\). For each configuration, we reconstruct validation images and compute the maximum absolute Pearson correlation over the eight square-grid ambiguities:
\[
    r_i
    =
    \max_{g\in D_4}
    \left|
    \operatorname{corr}
    \left(
    \mathbf f_i,\,
    g\cdot \widehat{\mathbf f}_i
    \right)
    \right|.
\]
The final Isomap MNIST score is the mean of this quantity across validation samples.

\paragraph{Neural manifold alignment.}
For the neural visual-coding experiment, Isomap is used as a two-dimensional manifold embedding baseline. We fit Isomap to the neural response vectors and obtain
\[
    \mathbf z_i
    =
    \operatorname{Isomap}(\mathbf x_i)
    \in \mathbb R^2 .
\]
As in the UMAP baseline, we evaluate whether the resulting embedding organizes neural responses according to the latent grating orientation. We center the two-dimensional embedding, extract a polar phase from each embedded point,
\[
    \widehat{\theta}_i
    =
    \operatorname{atan2}
    \left(
    z_{i,2}-\bar z_2,\,
    z_{i,1}-\bar z_1
    \right)
    \quad \operatorname{mod} 2\pi ,
\]
and compare these phases to the true doubled grating angles using circular correlation. Since the orientation and handedness of the learned manifold are arbitrary, we report the better of the original and mirrored circular correlations:
\[
    \max
    \left\{
    \rho_{\mathrm{circ}}(\widehat{\theta},\theta),
    \rho_{\mathrm{circ}}(2\pi-\widehat{\theta},\theta)
    \right\}.
\]
For Isomap, the only swept hyperparameter in the neural experiment is the neighborhood size,
\[
    k\in\{5,15,30,50\}.
\]
Unlike UMAP, Isomap does not use a minimum-distance parameter. The best configuration is selected by the highest absolute circular correlation.

\subsection{Topographic Independent Component Analysis (TICA) baseline}\label{baselines-tica}

To contextualize the performance of our proposed method, we evaluated 1D Topographic Independent Component Analysis (TICA) 
as a classical representation learning baseline. TICA extends standard ICA by modeling the dependencies between nearby components, 
theoretically making it a strong candidate for discovering locally correlated latent structures like those in the 
Generalized Shot Noise (GSN) dataset. 

\subsubsection*{Preprocessing and ZCA Whitening}
Standard ICA and TICA models assume that the input data is centered and spatially decorrelated. 
To provide the baseline with the optimal mathematical conditions for convergence, we strictly adhered to this prerequisite. 
Before training, we pre-generated the entire dataset into memory and computed the exact Zero-Phase Component Analysis (ZCA) 
whitening matrix. 

The dataset was first mean-centered: $X_c = X - \mu$. We then computed the covariance matrix and its eigendecomposition 
to formulate the whitening matrix:
$$W_{ZCA} = V (S + \epsilon I)^{-0.5} V^T$$
where $V$ and $S$ are the eigenvectors and eigenvalues of the covariance matrix, and $\epsilon = 10^{-5}$ is a small 
regularization constant. The entire dataset was projected through $W_{ZCA}$ prior to being fed into the TICA model, 
ensuring the features were fully decorrelated and unit-variance.

\subsubsection*{Experimental setup and hyperparameter grid}
Unlike our periodic graph-matching baselines, the TICA model and its corresponding dataset were strictly configured in 
the \textbf{aperiodic} regime, modeling standard bounded 1D sequences. Furthermore, the 
unmixing matrix was strictly constrained to be orthogonal (\texttt{use\_orthogonal = True}). 

We conducted a comprehensive grid search across 24 distinct hyperparameter configurations to ensure the model's performance 
was not an artifact of poor tuning. The 1D TICA model was instantiated with $d=63$ components. We optimized the network 
using Adam with global gradient clipping set to 1.0. All experiments were trained extensively for 1,500 epochs, 
with 1,000 steps per epoch and a batch size of 250.

The search space explored the following varying parameters:
\begin{itemize}
    \item \textbf{Learning Rate (LR):} $\{10^{-2}, 10^{-3}, 10^{-4}\}$
    \item \textbf{Topographic Pool Size:} $\{3, 5\}$
    \item \textbf{Feature Type:} Gaussian vs.\ Legendre polynomials
    \item \textbf{Output Representation:} Identity (Natural) vs.\ Discrete Sine Transform (DST-I)
\end{itemize}

\subsubsection*{Evaluation protocol and results}
To evaluate the quality of the learned representations, we measured the mean Pearson correlation coefficient 
between the ground truth hidden signals and the TICA model outputs. Because unsupervised models arbitrarily 
permute and scale latents, we utilized a highly generous evaluation protocol. For every batch, 
the evaluation script exhaustively searched through the entire symmetry group—testing all possible spatial shifts 
($d=63$), reflections (left/right flips), and sign inversions to find the absolute best possible 
alignment between the predictions and the ground truth.

Despite optimally whitening the inputs, restricting the model to orthogonal transforms, 
extensively sweeping the hyperparameters, and actively aligning the outputs during evaluation, 
TICA failed to recover the hidden signals. As shown in Table~\ref{tab:correlation_summary}, 
the absolute best-performing configuration yielded a maximum correlation of just $0.0725$. Scores of 
this magnitude indicate that the baseline essentially output random noise relative to the true targets, 
wholly failing to disentangle the overlapping, continuous topographic structure of the data.

\begin{table}[htbp]
    \centering
    \caption{Evaluation Summary of 1D TICA Representations}
    \label{tab:correlation_summary}
    \begin{tabular}{l c c l l c}
        \toprule
        \textbf{Experiment No} & \textbf{LR} & \textbf{Pool Size} & \textbf{Feature Type} & \textbf{Output Rep} & \textbf{Correlation} \\
        \midrule
        1       & 0.01   & 5 & Gaussian & DST-I     & 0.0725 \\
        2       & 0.01   & 3 & Gaussian & DST-I     & 0.0672 \\
        3       & 0.01   & 3 & Gaussian & Identity  & 0.0662 \\
        4       & 0.01   & 5 & Gaussian & Identity  & 0.0629 \\
        5       & 0.0001 & 3 & Gaussian & DST-I     & 0.0598 \\
        6       & 0.001  & 3 & Gaussian & DST-I     & 0.0563 \\
        7       & 0.001  & 5 & Gaussian & DST-I     & 0.0521 \\
        8       & 0.001  & 5 & Gaussian & Identity  & 0.0520 \\
        9       & 0.0001 & 5 & Gaussian & DST-I     & 0.0517 \\
        10      & 0.0001 & 3 & Gaussian & Identity  & 0.0493 \\
        11      & 0.001  & 3 & Gaussian & Identity  & 0.0484 \\
        12      & 0.0001 & 5 & Gaussian & Identity  & 0.0450 \\
        13      & 0.0001 & 3 & Legendre & Identity  & 0.0446 \\
        14      & 0.001  & 3 & Legendre & Identity  & 0.0412 \\
        15      & 0.001  & 3 & Legendre & DST-I     & 0.0358 \\
        16      & 0.001  & 5 & Legendre & DST-I     & 0.0357 \\
        17      & 0.0001 & 5 & Legendre & Identity  & 0.0352 \\
        18      & 0.0001 & 5 & Legendre & DST-I     & 0.0336 \\
        19      & 0.01   & 5 & Legendre & DST-I     & 0.0316 \\
        20      & 0.01   & 5 & Legendre & Identity  & 0.0303 \\
        21      & 0.001  & 5 & Legendre & Identity  & 0.0288 \\
        22      & 0.01   & 3 & Legendre & Identity  & 0.0282 \\
        23      & 0.01   & 3 & Legendre & DST-I     & 0.0275 \\
        24      & 0.0001 & 3 & Legendre & DST-I     & 0.0256 \\
        \bottomrule
    \end{tabular}
\end{table}

\subsection{LieGAN}\label{baselines-liegan}
To ensure a rigorous comparison, we adhered to the optimal operating configurations for both our method 
and the baseline (LieGAN~\cite{yang2023latent}).

\paragraph{Optimization budget and batch size.} 
Both frameworks were trained for a fixed optimization budget of approximately 150,000 iterations (gradient updates), 
and with the same dataset size of $50k$ samples. We utilized a batch size of $B=64$ for LieGAN, which is the default 
value while our method employed a batch size of $B=500$ to ensure low-variance estimation information theoretic quantities.

We prioritize number of \textit{gradient updates} over \textit{epochs}, as GAN convergence is 
primarily dictated by the number of generator-discriminator interaction steps. The 150k step budget significantly exceeds 
the typical convergence requirements for LieGAN on low-dimensional data ($d \le 15$), ensuring that under-training is 
not a factor in the reported results.

\paragraph{Hyperparameter search.}
We conducted a logarithmic grid search to tune the baseline for the specific datasets. The search space included:
\begin{itemize}
    \item Generator Learning Rate: $\{10^{-4}, 10^{-3}, 10^{-2}\}$
    \item Discriminator Learning Rate: $\{2\cdot 10^{-5}, 2\cdot 10^{-4}, 2\cdot 10^{-3}\}$
    \item Penalty Weight ($\lambda$): $\{0.1, 1.0, 10.0\}$
\end{itemize}
All baseline models were optimized using Adam with $\beta_1=0.5, \beta_2=0.999$.

\paragraph{Model selection.}
Given the unsupervised and unstable nature of adversarial training, we employed a rigorous model selection protocol 
that does not rely on ground-truth access. We saved checkpoints every 10,000 iterations and retrospectively selected the 
checkpoint that achieved the lowest internal Generator Loss (Adversarial + Penalty). The metrics reported in 
Table~\ref{tab:waveform_mnist_benchmarks} correspond to this best-performing checkpoint.

\paragraph{Data generation.}
To compare the symmetry discovery performance of our method against the well-established baseline LieGAN~\cite{yang2023latent}, 
we simulated the Generalized Shot Noise process on a grid of size $15$ across three distinct settings:

\begin{itemize}[leftmargin=*, noitemsep, topsep=2pt]
    \item \textbf{1. Data generated via $15$ translations representations:} We uniformly 
    sampled the width $\sigma$ of Gaussian pulses between $[0.5, 2.5]$, fixing their centers to the origin. 
    Then we discretized the pulses and introduced distribution symmetry by applying random cyclic shifts 
    using an FFT-based translation operator. Finally, we superposed the resultant vectors. 
    \item \textbf{2. Infinite-Dimensional Continuous Translations:} We sampled widths $\sigma$ and centers $\mu$ uniformly 
    among continuous values while ensuring periodic boundary conditions. Then we discretized the pulses and superposed the 
    resulting vectors.
    \item \textbf{3. Finite-Dimensional Discrete Translations:} We followed the procedure in case (1), 
    however, we sampled random shifts from set of integers instead of continuous values. 
\end{itemize}
\section{Morris sensitivity analysis}
\label{sec:morris_analysis}

To rigorously evaluate the robustness of the proposed framework 
against hyperparameter variations, we conducted a Morris elementary effects sensitivity analysis. 
The Morris method provides a computationally efficient way to screen for the most influential 
parameters and identify non-linear interactions across the hyperparameter space.

\subsection{Experimental setup}

We analyzed the five most critical hyperparameters governing our optimization dynamics 
and objective function trade-offs:
\begin{enumerate}
    \item \textbf{Estimator LR:} The learning rate of the stationarity estimator network. Bounds: $[0.0025, 0.0075]$.
    \item \textbf{Model LR:} The learning rate of the main model. Bounds: $[0.00025, 0.00075]$.
    \item \textbf{Stationarity:} The regularization coefficient for the stationarity maximization objective. Bounds: $[1.0, 1.66]$.
    \item \textbf{Resolution:} The regularization coefficient for the total correlation minimization objective. Bounds: $[0.75, 1.25]$.
    \item \textbf{InfoMax:} The regularization coefficient for the joint entropy maximization objective. Bounds: $[0.5625, 0.9375]$.
\end{enumerate}

Using 6 trajectories and an 8-level grid structure, we generated 36 distinct experimental configurations. 
Each experiment was trained for 1500 epochs with a batch size of 500.

We evaluate the framework's sensitivity using two primary metrics:
\begin{itemize}
    \item \textbf{Input-Output Correlation:} Measures the maximal correlation between 
    the ground-truth hidden signal and the recovered representation.
    \item \textbf{Band-Limited Generator Similarity:} Computes the cosine similarity between 
    the learned shift operator and the ideal translation generator after applying a band-limiting 
    DFT projector (norm = $0.75$).
\end{itemize}

\subsection{Global robustness}

As demonstrated in Table~\ref{tab:morris_raw_metrics}, our method exhibits 
remarkable stability across the diverse configurations sampled by the Morris trajectories. 
Across all 36 runs, the Input-Output Correlation ranges narrowly between $0.9682$ and $0.9883$, 
while the Band-Limited Generator Similarity consistently remains above $0.9897$ 
(with the vast majority exceeding $0.995$). This confirms that the proposed approach does not rely on a brittle, 
heavily tuned hyperparameter configuration to achieve successful latent domain recovery. 

\begin{table*}[htbp]
\centering
\caption{Raw Metrics for all 36 Morris Sensitivity Experiments.}
\label{tab:morris_raw_metrics}
\resizebox{\textwidth}{!}{%
\begin{tabular}{cccccccc}
\toprule
\textbf{Exp ID} & \textbf{Estimator LR} & \textbf{Model LR} & \textbf{Stationarity} & \textbf{Resolution} & \textbf{InfoMax} & \textbf{Correlation} & \textbf{Gen Similarity} \\
\midrule
1 & 0.0039 & 0.0007 & 1.5657 & 1.1786 & 0.7232 & 0.9851 & 0.9975 \\
2 & 0.0039 & 0.0007 & 1.5657 & 0.8929 & 0.7232 & 0.9869 & 0.9990 \\
3 & 0.0068 & 0.0007 & 1.5657 & 0.8929 & 0.7232 & 0.9872 & 0.9951 \\
4 & 0.0068 & 0.0007 & 1.1886 & 0.8929 & 0.7232 & 0.9869 & 0.9986 \\
5 & 0.0068 & 0.0007 & 1.1886 & 0.8929 & 0.9375 & 0.9856 & 0.9970 \\
6 & 0.0068 & 0.0004 & 1.1886 & 0.8929 & 0.9375 & 0.9857 & 0.9985 \\
7 & 0.0061 & 0.0003 & 1.2829 & 0.9643 & 0.5625 & 0.9845 & 0.9961 \\
8 & 0.0032 & 0.0003 & 1.2829 & 0.9643 & 0.5625 & 0.9849 & 0.9985 \\
9 & 0.0032 & 0.0006 & 1.2829 & 0.9643 & 0.5625 & 0.9849 & 0.9992 \\
10 & 0.0032 & 0.0006 & 1.6600 & 0.9643 & 0.5625 & 0.9845 & 0.9992 \\
11 & 0.0032 & 0.0006 & 1.6600 & 1.2500 & 0.5625 & 0.9836 & 0.9991 \\
12 & 0.0032 & 0.0006 & 1.6600 & 1.2500 & 0.7768 & 0.9846 & 0.9967 \\
13 & 0.0068 & 0.0004 & 1.4714 & 1.2500 & 0.9375 & 0.9849 & 0.9977 \\
14 & 0.0068 & 0.0004 & 1.4714 & 1.2500 & 0.7232 & 0.9848 & 0.9897 \\
15 & 0.0068 & 0.0004 & 1.0943 & 1.2500 & 0.7232 & 0.9848 & 0.9991 \\
16 & 0.0068 & 0.0007 & 1.0943 & 1.2500 & 0.7232 & 0.9852 & 0.9973 \\
17 & 0.0068 & 0.0007 & 1.0943 & 0.9643 & 0.7232 & 0.9868 & 0.9953 \\
18 & 0.0039 & 0.0007 & 1.0943 & 0.9643 & 0.7232 & 0.9862 & 0.9991 \\
19 & 0.0039 & 0.0008 & 1.2829 & 1.0357 & 0.6696 & 0.9851 & 0.9990 \\
20 & 0.0068 & 0.0008 & 1.2829 & 1.0357 & 0.6696 & 0.9853 & 0.9949 \\
21 & 0.0068 & 0.0005 & 1.2829 & 1.0357 & 0.6696 & 0.9862 & 0.9977 \\
22 & 0.0068 & 0.0005 & 1.6600 & 1.0357 & 0.6696 & 0.9854 & 0.9987 \\
23 & 0.0068 & 0.0005 & 1.6600 & 1.0357 & 0.8839 & 0.9831 & 0.9987 \\
24 & 0.0068 & 0.0005 & 1.6600 & 0.7500 & 0.8839 & 0.9682 & 0.9973 \\
25 & 0.0075 & 0.0006 & 1.0943 & 0.8214 & 0.7768 & 0.9883 & 0.9988 \\
26 & 0.0075 & 0.0003 & 1.0943 & 0.8214 & 0.7768 & 0.9839 & 0.9969 \\
27 & 0.0046 & 0.0003 & 1.0943 & 0.8214 & 0.7768 & 0.9881 & 0.9993 \\
28 & 0.0046 & 0.0003 & 1.0943 & 1.1071 & 0.7768 & 0.9851 & 0.9967 \\
29 & 0.0046 & 0.0003 & 1.4714 & 1.1071 & 0.7768 & 0.9852 & 0.9979 \\
30 & 0.0046 & 0.0003 & 1.4714 & 1.1071 & 0.5625 & 0.9825 & 0.9968 \\
31 & 0.0025 & 0.0008 & 1.1886 & 1.1786 & 0.6161 & 0.9842 & 0.9990 \\
32 & 0.0025 & 0.0008 & 1.1886 & 0.8929 & 0.6161 & 0.9854 & 0.9968 \\
33 & 0.0025 & 0.0008 & 1.1886 & 0.8929 & 0.8304 & 0.9873 & 0.9987 \\
34 & 0.0025 & 0.0005 & 1.1886 & 0.8929 & 0.8304 & 0.9843 & 0.9974 \\
35 & 0.0054 & 0.0005 & 1.1886 & 0.8929 & 0.8304 & 0.9869 & 0.9935 \\
36 & 0.0054 & 0.0005 & 1.5657 & 0.8929 & 0.8304 & 0.9847 & 0.9991 \\
\bottomrule
\end{tabular}%
}
\end{table*}

\subsection{Elementary effects results}

To quantify the specific influence of each parameter, we extract the Morris elementary effects: 
the mean effect ($\mu$), the absolute mean effect ($\mu^*$), and the standard deviation of the effect 
($\sigma$). The $\mu^*$ value indicates the overall influence of the parameter on the final metric, 
while $\sigma$ captures non-linear effects and interactions with other hyperparameters.

As shown in Table~\ref{tab:morris_analysis}, for Input-Output Correlation, 
the Total Correlation coefficient (\textit{Resolution}) proves to be the most sensitive parameter 
($\mu^* = 0.0059$, $\sigma = 0.0102$). For Band-Limited Generator Similarity, both the 
\textit{Stationarity} coefficient ($\mu^* = 0.0052$) and the \textit{Estimator LR} ($\mu^* = 0.0051$) 
exert the most influence. However, it is vital to contextualize these findings with the absolute bounds observed in 
Table~\ref{tab:morris_raw_metrics}: because the global variance of the metrics is incredibly tight, 
even the parameters marked as ``most influential'' by the Morris analysis cause negligible 
absolute performance degradation.

\begin{table*}[htbp]
\centering
\caption{Morris sensitivity analysis metrics for correlation and generator similarity.}
\label{tab:morris_analysis}
\begin{tabular}{lcccccc}
\toprule
& \multicolumn{3}{c}{\textbf{Input-Output Correlation}} & \multicolumn{3}{c}{\textbf{Generator Similarity}} \\
\cmidrule(lr){2-4} \cmidrule(lr){5-7}
\textbf{Parameter} & $\mu$ & $\mu^*$ & $\sigma$ & $\mu$ & $\mu^*$ & $\sigma$ \\
\midrule
Estimator LR & -0.0003 & 0.0021 & 0.0034 & -0.0051 & 0.0051 & 0.0012 \\
Model LR & 0.0017 & 0.0022 & 0.0031 & -0.0005 & 0.0025 & 0.0029 \\
Stationarity & -0.0008 & 0.0009 & 0.0014 & -0.0012 & 0.0052 & 0.0077 \\
Resolution & 0.0015 & 0.0059 & 0.0102 & 0.0003 & 0.0024 & 0.0029 \\
InfoMax & 0.0005 & 0.0023 & 0.0028 & 0.0018 & 0.0038 & 0.0056 \\
\bottomrule
\end{tabular}
\end{table*}
\section{Caveats}
\paragraph{Initialization of eigenvalue phases.}
We initialize the norm of Lie basis eigenvalues using a normal distribution with
\( \sigma = 10^{-3} \) and \(\mu = 0\). Using smaller standard deviations did not affect performance, 
however, significantly larger $\sigma$ values may lead to corrupt the optimization 
trajectory.

\paragraph{Noise injection to the resolution filter.}
We initialize the resolving filter with zeros and add Gaussian noise during the early stages of training before 
computing the loss for each batch.
\begin{equation}
    \mathbf{\psi} \leftarrow \mathbf{\psi} + \mathcal{N}(\mu=0, \sigma = 0.1)  \exp(-(t/\tau)) 
\end{equation}
The amplitude of the noise is set to decay exponentially 
with a short time constant (of $\tau = 10$ epochs). As mentioned previously, 
to compute the transformed data $\mathbf{y}$, we use a normalized version of $\filter$ at each step:
$\mathbf{\hat{\psi}} = {\mathbf{\psi}}/{||\mathbf{\psi}||_2}$.

\end{document}